\documentclass[a4paper,11pt]{article}

\usepackage[margin=1.2in]{geometry}

\usepackage{amsmath}
\usepackage{algorithmic}
\usepackage{algorithm}
\usepackage{amsthm}
\usepackage{amsfonts}
\usepackage{comment}
\usepackage{graphicx}
\usepackage{hyperref}
\usepackage{color}
\usepackage{subcaption}

\newtheorem{theorem} {Theorem}
\newtheorem{lemma} {Lemma}

\newtheorem{assumption} {Assumption}

\def\q{{\mathbf{q}}}
\def\x{{\mathbf{x}}}

\def\v{{\mathbf{v}}}
\def\z{{\mathbf{z}}}
\def\w{{\mathbf{w}}}
\def\y{{\mathbf{y}}}

\def\X{{\mathbf{X}}}

\def\M{{\mathbf{M}}}

\def\V{{\mathbf{V}}}

\def\W{{\mathbf{W}}}

\def\Q{{\mathbf{Q}}}

\def\D{{\mathbf{D}}}

\newcommand{\mS}{\mathcal{S}}

\newcommand{\mD}{\mathcal{D}}
\newcommand{\E}{\mathbb{E}}

\newcommand{\trace}{\textrm{Tr}}

\newcommand{\reals}{\mathbb{R}}

\title{On the Regret Minimization of Nonconvex Online Gradient Ascent for Online PCA}
\date{}
\author{Dan Garber \\
Technion - Israel Institute of Technology \\
{\small{dangar@technion.ac.il}}}

\begin{document}

 \maketitle

\begin{abstract}
Non-convex optimization with global convergence guarantees is gaining significant interest in machine learning research in recent years. However, while most works consider either offline settings in which all data is given beforehand, or simple online stochastic i.i.d. settings, very little is known about non-convex optimization for adversarial online learning settings.
In this paper we focus on the problem of Online Principal Component Analysis in the regret minimization framework. For this problem, all existing regret minimization algorithms for the fully-adversarial setting are based on a positive semidefinite convex relaxation, and hence require quadratic memory and SVD computation (either thin of full) on each iteration, which amounts to at least quadratic runtime per iteration.
This is in stark contrast to a corresponding stochastic i.i.d. variant of the problem, which was studied extensively lately, and admits very efficient gradient ascent algorithms that work directly on the natural non-convex formulation of the problem, and hence require only linear memory and linear runtime per iteration. This raises the question: \textit{can non-convex online gradient ascent algorithms be shown to minimize regret in online adversarial settings?}

In this paper we take a step forward towards answering this question. We introduce an \textit{adversarially-perturbed spiked-covariance model} in which, each data point is assumed to follow a fixed stochastic distribution with a non-zero spectral gap in the covariance matrix,  but is then perturbed with some adversarial vector. This model is a natural extension of a well studied standard \textit{stochastic} setting that allows for non-stationary (adversarial) patterns to arise in the data  and hence, might serve as a significantly better approximation for real-world data-streams.
We show that in an interesting regime of parameters, when the non-convex online gradient ascent algorithm is initialized with a ``warm-start" vector, it provably minimizes the regret with high probability. We further discuss the possibility of computing such a ``warm-start" vector, and also the use of regularization to obtain fast regret rates. Our theoretical findings are supported by empirical experiments on both synthetic and real-world data.
\end{abstract}

\section{Introduction}
Nonconvex optimization is ubiquitous in contemporary machine learning, ranging from optimization over sparse vectors or low-rank matrices to training Deep Neural Networks. While traditional (yet still highly active) research on nonconvex optimization focuses mostly on efficient convergence to stationary points, which in general need not even be a local minima, let alone a global one (e.g., \cite{Agarwal17, Natasha2, allen2016variance, CarmonDHS17}), a more-recent line of work focuses on proving convergence to global minima, usually under certain simplifying assumptions that on one hand make the nonconvex problem tractable, and on the other hand, are sufficiently reasonable in some scenarios of interest (see for instance \cite{de2015global, ge2016matrix, bhojanapalli2016, arora2014provable, jin2016provable} to name only a few).
One of the most studied and well known nonconvex optimization problems in machine learning underlies the fundamental task of \textit{Principal Component Analysis} (PCA) \cite{Pearson, Hotelling, Jolliffe2011}, in which, given a set of $N$ vectors in $\reals^d$, one wishes to find a $k$-dimensional subspace for $k<<d$, such that the projections of these vectors onto this subspace is closest in square-error to the original vectors. It is well known that the optimal subspace corresponds to the span of the top $k$ eigenvectors of the covariance matrix of the data-points. Henceforth, we focus our discussion to the case $k=1$, i.e., extracting the top principal component.
Quite remarkably, while this problem is non-convex (since extracting the top eigenvector amounts to \textit{maximizing} a convex function over the unit Euclidean ball), a well known iterative algorithm known as \textit{Power Method} (or Power Iterations, see for instance \cite{golub2012matrix}), which simply starts with a random unit vector and repeatedly applies the covariance matrix to it (and then normalizes the result to have unit norm), converges to the global optimal solution rapidly. The convergence guarantee of the PM, can also be shown to imply that the nonconvex projected gradient ascent method with random initialization and a fixed step-size also converges to the top principal component. \footnote{This follows since the steps of PGA can be rewritten as applying PM steps to a scaled and shifted version of the original matrix that preserves the leading eigenvector.}

In a recent line of work, the convergence of non-convex gradient methods for PCA was extended to a natural online stochastic i.i.d. setting of the problem, in which, given a stream of data points sampled i.i.d. from a fixed (unknown) distribution, the goal is to converge to the top eigenvector of the covariance  matrix of the underlying distribution as the sample size increases, yielding algorithms that require only linear memory (i.e., do not need to store the entire sample or large portions of it at any time) and linear runtime to process each data point, see for instance \cite{Mitliagkas13, Balsubramani13, Shamir16a, jain2016matching, Allen2017,li2018near, xu2018accelerated}.

In a second recent line of research, researchers have considered Online PCA as a sequential decision problem in the adversarial framework of regret minimization (aka online learning, see for instance the introductory texts \cite{CesaBook, Hazan16, SSS12}), e.g., \cite{Warmuth06a, Warmuth06b, Warmuth13, Dwork14, G15online, allen2017follow}. In this framework, for each data-point, the online algorithm is required to predict a unit vector (i.e., a subspace of dimension one, recall we are in the case $k=1$) \textit{before} observing the data-point, and the goal is to minimize regret which is the difference between the square-error of the predictions made and the square-error of the principal component of the entire sequence of data. Different from the i.i.d. stochastic setting, in this framework, the data may be completely arbitrary (though assumed to be bounded in norm), and need not follow a simple generative model. Formally, the regret is given by
\begin{eqnarray*}
\textrm{regret} := \sum_{i=1}^N\Vert{\x_i - \w_i\w_i^{\top}\x_t}\Vert_2^2 - \min_{\Vert{\w}\Vert_2=1}\sum_{i=1}^N\Vert{\x_i - \w\w^{\top}\x_t}\Vert_2^2,
\end{eqnarray*}
where $\{\x_i\}_{i\in[N]}\subset\reals^d$ is the sequence of data points, and $\{\w_i\}_{i\in[N]}$ is the sequence of predictions made by the online algorithm. Using standard manipulations, it can be shown that 
\begin{eqnarray*}
\textrm{regret} = \max_{\Vert{\w}\Vert_2=1}\sum_{i=1}^N(\w^{\top}\x_t)^2 - \sum_{i=1}^N(\w_i^{\top}\x_i)^2 =\lambda_1\left({\sum_{t=1}^N\x_i\x_i^{\top}}\right) - \sum_{t=1}^N(\w_i^{\top}\x_i)^2,
\end{eqnarray*}
where $\lambda_1(\cdot)$ denotes the largest (signed) eigenvalue of a real symmetric matrix.

Naturally, the arbitrary nature of the data in the online learning setting, makes the problem much more difficult than the stochastic i.i.d. setting. Notably, all current algorithms which minimize regret in this fully-adversarial setting cannot directly tackle the natural nonconvex formulation of the problem, but consider a well known (tight) convex relaxation, which ``lifts" the decision variable from the unit Euclidean ball in $\reals^d$ to the set of all $d\times d$ positive semidefinite matrices of unit trace (aka the spectrahedron). While this reformulation allows to obtain regret-minimizing algorithms in the online adversarial settings (since the problem becomes convex), they are dramatically less efficient than the standard nonconvex gradient methods. In particular, all such algorithms require quadratic memory (i.e., $O(d^2)$), and require either a thin or full-rank SVD computation of a full-rank matrix to process each data point, which amounts to at least quadratic runtime per data point (for non trivially-sparse data), see \cite{Warmuth06a, Warmuth06b, Warmuth13, Dwork14, G15online, allen2017follow}. This phenomena naturally raises the question:
\begin{center}
\textit{Can Nonconvex Online Gradient Ascent be shown to minimize regret for the Online PCA problem?}
\end{center}
While in this paper we do not provide a general answer (either positive or negative),  we do take a step forward towards understanding the applicability of nonconvex gradient methods to the Online PCA problem. We introduce a ``semi-adversarial" setting, which we refer to as \textit{adversarially-perturbed spiked-covariance model}, which assumes the data follows a standard i.i.d. stochastic distribution with a covariance matrix that admits a non-zero spectral gap, however, each data point is then perturbed by some arbitrary, possibly adversarial, vector of non-trivial magnitude. We view this model as a natural extension of the standard stochastic model (which was studied extensively in recent years, see references above) due to its ability to capture arbitrary (adversarial) patterns in the data. Hence, we believe  the suggested model might provide a much better approximation for real-world data streams.
We formally prove that in a certain regime of parameters, which concerns both the spectral properties of the distribution covariance and the magnitude of adversarial perturbations, given a ``warm-start" initialization which is sufficiently correlated with the top principal component of the stochastic distribution, the natural nonconvex online gradient ascent algorithm guarantees an $\tilde{O}(\sqrt{N})$ regret bound with high probability. In particular, the algorithm requires only $O(d)$ memory and $O(d)$ runtime per data point. We further discuss the possibilities of computing such a "warm-start" vector (i.e., initializing from a "cold-start"). Moreover, we explore the possibility of adding regularization to the algorithm, which as we show, under the same assumptions on the data and the same "warm-start" initialization, allows to obtain a $\textrm{poly}(\log{N})$ regret bound, still using only $O(d)$ memory and runtime per data point. \footnote{Note that in the fully-adversarial setting (i.e., there is no stochastic component), there is a  $\Omega(\sqrt{N})$ lower bound on the regret, see for instance \cite{Warmuth06b}.}

 Finally, we present empirical experiments with both synthetic and real-world datasets which complement our theoretical analysis.


\subsection{Related work}
Besides the works \cite{Warmuth06a, Warmuth06b, Warmuth13, Dwork14, G15online, allen2017follow} mentioned above, which consider online algorithms for the fully-adversarial setting, in a very recent work \cite{Arora18a}, the authors have introduced an online PCA setting in which the data-points (the vectors $\{\x_i\}_{i\in[n]}$) are drawn i.i.d. from a fixed (unknown) distribution, however the feedback observed by the algorithm on each round $i$ is a perturbed version $\x_i+\y_i$, where $\y_i$ is an arbitrary (e.g., adversarial) noise. The authors show that under the conditions $\max_{i\in[N]}\Vert{\x_i}\Vert_2 \leq 1$ and $\sum_{i=1}^N\Vert{\y_i}\Vert_2+\Vert{\y_i}\Vert_2^2 \leq \sqrt{N}$, the natural nonconvex online gradient ascent guarantees $O(\sqrt{N})$ regret with high probability (see Theorem 3.2 in \cite{Arora18a}) w.r.t. the original sequence $\{\x_i\}_{i\in[N]}$ (i.e., without the noise). While this model which combines stochastic and adversarial components is somewhat similar to ours, there are three major differences. First, in our setting the adversarial component is considered part of the data, and hence is included in the definition of the regret, while in \cite{Arora18a}, the adversarial component is considered only as noisy feedback to the algorithm and is not considered part of the data (and corresponding regret bound). Second, while in \cite{Arora18a} it is required that the sum of magnitudes of the adversarial components is only sublinear in the sequence length ($\sqrt{N}$, as discussed above), which intuitively makes these components negligible in the regret analysis, in this work, as we discuss in the next section, we allow each adversarial component to be of constant magnitude, independent of the sequence length, and hence overall, the noisy components are non negligible in the regret analysis. Third, while \cite{Arora18a} only gives an $O(\sqrt{N})$ regret bound, we show that with proper regularization (which does not change the overall complexity of the algorithm), a regret bound of $\textrm{poly}(\log{N})$ can be obtained.

In another recent work \cite{Arora18b}, the authors consider the application of the online mirror decent algorithm to the \textit{stochastic} online PCA problem (i.e., when the data points are sampled i.i.d. from a fixed unknown distribution). This mirror descent algorithm works on the convex semidefinite relaxation of the problem discussed above (and hence requires a potentially-expensive matrix factorization on each iteration). The authors show that under a spectral gap assumption in the distribution covariance matrix, adding strongly-convex regularization (w.r.t. the Euclidean norm) to the algorithm can improve the regret bound from $O(\sqrt{N})$ to $\textrm{poly}(\log{N})$. In this work we draw inspiration from this observation, and show that also in our model which combines stochastic data with adversarial data, the addition of a regularizing term can improve the regret bound of the nonconvex online gradient method from $\tilde{O}(\sqrt{N})$ to $\textrm{poly}(\log{N})$. 

\section{Assumptions and Results}

In this section we formally introduce our assumptions and main result. 
As discussed in the introduction, since our aim is make progress on a highly non-trivial problem of providing global convergence guarantees for a non-convex optimization algorithm in an online adversarial setting, our results do not hold for arbitrary (bounded) data, as is usually standard in \textit{convex} online learning settings, but only for a more restricted family of input streams, namely those which follow a model we refer to in this paper as the \textit{adversarially-perturbed spiked-covariance model}. Next we formally introduce this model.

\subsection{Adversarially-Perturbed Spiked-Covariance Model}
Throughout the paper we assume the data, i.e., the vectors $\{\x_t\}_{t\in[N]}$, satisfy the following assumption.

\begin{assumption}[Perturbed Spiked Covariance Model]\label{ass:dist}
We say a sequence of $N$ vectors $\{\x_t\}_{t\in[N]}\subset\reals^d$ satisfies Assumption \ref{ass:dist}, if for all $t\in[N]$, $\x_t$ can be written as $\x_t = \q_t + \v_t$, where $\{\q_t\}_{t\in[N]}$ are sampled i.i.d. from a distribution $\mD$ and  $\{\v_t\}_{t\in[N]}$ is a sequence of arbitrary bounded vectors such that the following conditions hold:
\begin{enumerate}
\item
the vectors $\{\v_t\}_{t\in[N]}$ all lie in a Euclidean ball of radius $V$ centered at the origin, i.e., $\max_{t\in[N]}\Vert{\v_t}\Vert_2 \leq V$
\item
the support of $\mD$ is contained in a Euclidean ball of radius $R$ centered at the origin, i.e., $\sup_{\q\in\textrm{support}(\mD)}\Vert{\q}\Vert_2\leq R$
\item
$\mD$ has zero mean, i.e., $\E_{\q\sim\mD}[\q] = \mathbf{0}$
\item
the covariance matrix $\Q:=\E_{\q\sim\mD}[\q\q^{\top}]$, admits an eigengap $\delta(\Q) := \lambda_1(\Q)-\lambda_2(\Q)$ which satisfies $\delta(\Q) \geq V\sqrt{2\lambda_1(\Q)+V^2} + \varepsilon$, for some $\varepsilon>0$.
\end{enumerate}
\end{assumption}

We now make a few remarks regarding Assumption \ref{ass:dist}. Item (1) assumes that the adversarial perturbations are bounded which is standard in the online learning literature, Item (2) is also a standard assumption, which is used to apply standard concentration arguments for sums of i.i.d random variables. Item (3), i.e., the assumption that the distribution as zero mean, while often standard, is not mandatory in general for our analysis technique to hold, however since it greatly simplifies the analysis we make it. 

To better understand Item (4), it helps to think of $\delta(\Q),V^2,\varepsilon$ as quantities proportional to $\lambda_1(\Q)$, i.e., consider $\delta(\Q) = c_{\delta}\lambda_1(\Q)$, $V^2 = c_V\lambda_1(\Q)$, $\varepsilon = c_{\varepsilon}\lambda_1(\Q)$, for some universal constants $c_{\delta},c_V,c_{\varepsilon}\in(0,1)$. Now, Item (4) in the assumption boils down the the condition
$c_{\delta} \geq \sqrt{2c_V + c_V^2} + c_{\varepsilon}$\footnote{
or alternatively,
$c_V \leq \sqrt{1 + (c_{\delta} - c_{\varepsilon})^2} - 1 \leq (c_{\delta} - c_{\varepsilon})^2/2$,
where the last inequality follows from the standard inequality $\sqrt{a+b} \leq \sqrt{a} + \frac{b}{2\sqrt{a}}$.}. That is, the eigengap in the covariance $\Q$ needs to dominate the adversarial perturbations in a certain way. Note that in principle, this regime of parameters still allows the ratio $V/R$ (i.e., ratio between maximal magnitude of adversarial component and maximal magnitude of stochastic component) to even be a universal constant. It is also important to note that the assumption of a non-negligeble eigengap in the covariance matrix is natural for PCA and is often observed in practice. We further discuss this assumption after presenting our main theorem - Theorem \ref{thm:main} in the following subsection.

\paragraph{Connection with stochastic i.i.d. models:} note that when setting $V = 0$ in Assumption \ref{ass:dist} (i.e., there is no adversarial component), our setting reduces to the well studied standard stochastic i.i.d. setting. In particular, in this case item (4) in Assumption \ref{ass:dist} simply reduces to the standard assumption in this model that the covariance admits an eigengap bounded away from zero ($\delta(\Q) \geq \varepsilon$). Hence, the model introduced above can be seen as a natural, yet highly non-trivial, extension of the standard stochastic model to a ``more expressive" online adversarial model, that might serve as a better approximation for real-world data-streams in online-computation environments. 

\subsection{Algorithm and Convergence Result}

For simplicity of the analysis we consider the data as arriving in blocks of length $\ell$, where $\ell$ is a parameter to be determined later. Towards this end, we assume that $N = T\ell$ for some integer $T$ and we consider prediction in $T$ rounds, such that on each round $t\in[T]$, the algorithm predicts on all $\ell$ vectors in the $t$th block, which we denote by $\x_t^{(1)},\dots,\x_t^{(\ell)}$. It is important to emphasize that, while our algorithm considers the original data in blocks, it requires only $O(d)$ memory and $O(d)$ time to process each data point $\x_t^{(i)}$ for any $t\in[T], i\in[\ell]$. 

Our algorithm, which we refer to as \textit{nonconvex online gradient ascent}, is given below (see Algorithm \ref{alg:opm}). Our algorithm comes in two variants, one without additional regularization and one with. As can be seen the non-regularized version is equivalent to applying the Online Gradient Ascent algorithm \cite{zinkevich03, Hazan16} with the payoff function $f_t(\hat{\w}) := \frac{1}{2}\sum_{i=1}^{\ell}(\w^{\top}\x_t^{(i)})^2$ on each round $t\in[T]$, where $\w$ is constrained to be a unit vector (though here we recall that the payoff function is not concave in $\w$, and the feasible set is not convex). 
The regularized version is similar, but considers the regularized payoff function $f^{\alpha}_t(\w) := \frac{1}{2}\sum_{i=1}^{\ell}(\w^{\top}\x_t^{(i)})^2 - \frac{\alpha}{2}\Vert{\w}\Vert_2^2$. \footnote{While at this point it may not be immediately clear how the introduction of the regularizer helps, since we are optimizing over the unit sphere (hence the regularizer has the same value for all feasible points), it will be apparent in the analysis that it allows to obtain faster rates, similarly to the way that adding a strongly convex regularizer enables to obtain faster rates in the standard setting of online convex optimization \cite{Hazan16}.}

\begin{algorithm}[H]
\caption{Nonconvex Online Gradient Ascent for Online PCA}
\label{alg:opm}
\begin{algorithmic}[1]
\STATE input: unit vector $\hat{\w}_1$, sequence of positive learning rates $\{\eta_t\}_{t\geq 1}$, regularization parameter $\alpha \geq 0$.
\FOR{$t=1\dots T$}
\STATE predict vector $\hat{\w}_t$
\STATE observe $\ell$ vectors $\x_t^{(1)},\dots,\x_t^{(\ell)}$ and payoff $\sum_{i=1}^{\ell}(\hat{\w}_t^{\top}\x_t^{(i)})^2$
\STATE compute the update
\begin{eqnarray*}
&\hat{\w}_{t+1} \gets \frac{\hat{\w}_t + \eta_t\sum_{i=1}^{\ell}\x_t^{(i)}\x_t^{(i)\top}\hat{\w}_t}{\Vert{\hat{\w}_t + \eta_t\sum_{i=1}^{\ell}\x_t^{(i)}\x_t^{(i)\top}\hat{\w}_t}\Vert_2} &\quad \{\textrm{without regularization}\} \\
& \textrm{OR} &\\
&\hat{\w}_{t+1} \gets \frac{(1-\eta_t\alpha)\hat{\w}_t + \eta_t\sum_{i=1}^{\ell}\x_t^{(i)}\x_t^{(i)\top}\hat{\w}_t}{\Vert{(1-\eta_t\alpha)\hat{\w}_t + \eta_t\sum_{i=1}^{\ell}\x_t^{(i)}\x_t^{(i)\top}\hat{\w}_t}\Vert_2} &\quad \{\textrm{with regularization}\} 
\end{eqnarray*}
\ENDFOR
\end{algorithmic}
\end{algorithm}

The following two theorem state our main results.

\begin{theorem}\label{thm:main}[convergence of Algorithm \ref{alg:opm} without regularization ($\alpha =0$) and constant learning rate]
Consider a sequence of vectors $\{\x_t\}_{t\in[N]}$ which follows Assumption \ref{ass:dist} and fix $p\in(0,1)$. For $N$ large enough, there exists an integer $\ell = O\left({\frac{R^4\lambda_1(\Q)^2}{\left({\delta(\Q)^2 - V^2(2\lambda_1(\Q)+V^2)}\right)^2}\log\frac{dN}{p}}\right)$,
such that applying Algorithm \ref{alg:opm} with blocks of length $\ell$ and initialization $\hat{\w}_1$ which satisfies 
\begin{eqnarray*}
(\hat{\w}_1^{\top}\x)^2 \geq 1- \frac{\delta(\Q) - V^2}{2\lambda_1(\Q)}\left({1-\frac{\delta(\Q)}{9\lambda_1(\Q)}}\right),
\end{eqnarray*}
where $\x$ is the leading eigenvector of $\Q$ (as defined in Assumption \ref{ass:dist}), and with a constant learning rate $\eta_t = \eta = \frac{1}{\sqrt{T}\ell(R+V)^2}$ for all $t\in[T]$ and without regularization (i.e., $\alpha=0$), guarantees that with probability at least $1-p$, the regret is upper-bounded by $$O\left({\sqrt{N\log\frac{dN}{p}}\frac{\lambda_1(\Q)R^4}{\delta(\Q)^2-V^2(V^2+\lambda_1(\Q))}}\right).$$
\end{theorem}

Theorem \ref{thm:main} roughly says that when the distribution covariance has a large-enough eigengap with respect to the adversarial perturbations (item 4 in Assumption \ref{ass:dist}), then non-convex OGA converges from a ``warm-start" with $\tilde{O}(\sqrt{N})$ regret. Intuitively, the condition on the eigengap implies that the best-in-hindsight eigenvector cannot be far from $\x$ - the leading eigenvector of the distribution covariance by more than a certain constant. Hence, Theorem \ref{thm:main} can be seen as an online ``local" convergence result. Importantly, it is not hard to show that under the conditions of the theorem, the best-in-hindsight eigenvector can also be far from both the initial vector $\hat{\w}_1$ and from $\x$ by a constant (and hence in particular both $\hat{\w}_1$ and $\x$ can incur linear regret). Hence, while our setting is strictly easier than the fully-adversarial online learning setting, it still a highly non-trivial online learning setting. In particular, all previous algorithms that provably minimize the regret under the conditions of Theorem \ref{thm:main} require quadratic memory and quadratic runtime per data-point.

Our second main result shows that the regularized version of Algorithm \ref{alg:opm} (i.e., with $\alpha >0$) can guarantee poly-logarithmic regret in $N$ under Assumption \ref{ass:dist}.

\begin{theorem}\label{thm:mainLog}[convergence of Algorithm \ref{alg:opm} with regularization ($\alpha > 0$)]
Consider a sequence of vectors $\{\x_t\}_{t\in[N]}$ which follows Assumption \ref{ass:dist} and fix $p\in(0,1)$. For $N$ large enough, there exists an integer $\ell = O\left({\frac{R^4\lambda_1(\Q)^2}{\left({\delta(\Q)^2 - V^2(\lambda_1(\Q)+V^2)}\right)^2}\log\frac{dN}{p}}\right)$, such that applying Algorithm \ref{alg:opm} with blocks of length $\ell$, regularization parameter which satisfies
\begin{eqnarray*}
\alpha = \Theta\left({\frac{\ell}{\delta(\Q)+V^2}\left({\delta(\Q)^2-V^4-2V^2\lambda_1(\Q)}\right)}\right), 
\end{eqnarray*}
and initialization $\hat{\w}_1$ which satisfies 
\begin{eqnarray*}
(\hat{\w}_1^{\top}\x)^2 \geq 1- \frac{\delta(\Q) - V^2}{2\lambda_1(\Q)}\left({\frac{9}{10}-\frac{\delta(\Q)}{9\lambda_1(\Q)}}\right),
\end{eqnarray*}
where $\x$ is the leading eigenvector of $\Q$ (as defined in Assumption \ref{ass:dist}), and with learning rate $\eta_t =  \frac{1}{\alpha{}t+T_0}$ for some $T_0\geq 0$ large enough, guarantees that with probability at least $1-p$, the regret is upper-bounded by $$O\left({\frac{R^8(\delta(\Q)+V^2)\lambda_1(\Q)^2}{\left({\delta(\Q)^2-V^4-2\lambda_1(\Q)V^2}\right)^3}\log(N)\log\left({\frac{dN}{p}}\right)}\right).$$
\end{theorem}

It is important to note that the ability to obtain a poly-logarithmic regret bound as Theorem \ref{thm:mainLog} suggests, relies crucially on the eigen-gap assumption in the stochastic covariance matrix in Assumption \ref{ass:dist}. In particular, for the standard fully-adversarial version of the problem (i.e., there is no stochastic component in the data vectors) it is not possible in worst case to improve over the known $O(\sqrt{N})$ bound (see \cite{Warmuth06b}). We see this fact as further motivation for studying intermediate models that bridge between the fully-adversarial and fully-stochastic settings, showing that such models can result in much improved convergence guarantees compared to the possibly over-pessimistic fully-adversarial setting.

\subsubsection{Computing a "warm-start" vector} 
We now discuss the possibility of satisfying the "warm-start" requirement in Theorem \ref{thm:main}.

First, we note that  given the possibility to sample i.i.d. points from the underlying distribution $\mD$, it is straightforward to obtain a warm-start vector $\hat{\w}_1$, as required by Theorem \ref{thm:main}, by simply initializing $\hat{\w}_1$ to be the leading eigenvector of the empirical covariance of a size-$n$ sample of such points. It is not difficult to show via standard tools such as the Davis-Kahan $\sin\theta$ theorem and a Matrix-Hoeffding concentration inequality (see for instance the proof of the following Lemma \ref{lem:warmstart}), that for any $(\epsilon,p)\in(0,1)^2$, a sample of size $n = O\left({ \frac{R^4\ln(d/p)}{\epsilon\delta(\Q)^2}}\right)$ suffices, so the outcome $\hat{\w}_1$ satisfies: $(\hat{\w}_1^{\top}\x)^2 \geq 1-\epsilon$ with probability at least $1-p$.

If sampling directly from $\mD$ is not possible, the following lemma, whose proof is given in the sequel, shows that with a simple additional assumption on the parameters $\delta(\Q),\lambda_1(\Q),V^2$, it is possible to obtain the warm-start initialization directly using data that follows Assumption \ref{ass:dist}. Moreover, the sample-size $n$ required is independent of the sequence length $N$, and hence using for instance the first $n$ vectors in the stream to compute such initialization, deteriorates the overall regret bound in Theorems \ref{thm:main}, \ref{thm:mainLog} only by a low-order term.

\begin{lemma}\label{lem:warmstart}[warm-start]
Fix some $c\in(0,1]$ and suppose that in addition to Assumption \ref{ass:dist} it also holds that $\delta(\Q) \geq \left({32c^{-1}\lambda_1(\Q)V^4}\right)^{1/3}$. Then, for any $p\in(0,1)$ there exists a sample size $n =O\left({\frac{R^4\lambda_1\log(d/p)}{c\delta(\Q)^3}}\right)$, such that initializing $\hat{\w}_1$ to be the leading eigenvector of the empirical covariance $\hat{\X} := \frac{1}{n}\sum_{i=1}^n\x_i\x_i^{\top}$, where $\x_1,\dots,\x_n$ follow Assumption \ref{ass:dist}, guarantees that with probability at least $1-p$: $(\hat{\w}_1^{\top}\x)^2 \geq 1- c\frac{\delta(\Q) - V^2}{2\lambda_1(\Q)}$.
\end{lemma}

\section{Analysis}

At a high-level, the proof of Theorems \ref{thm:main}, \ref{thm:mainLog} relies on the combination of the following three ideas:
\begin{enumerate}
\item
We build on the fact that the Online PCA problem, when cast as online \textit{linear} optimization over the spectraedron (i.e., when the decision variable is lifted from a unit vector to a positive semidefinite matrix of unit trace), is  online learnable via a standard application of online gradient ascent, which achieves an $O(\sqrt{N})$ regret bound in the non-regularized case and $O(\log{N})$ regret bound with additional regularization (note however that as discussed above, this approach requires a full SVD computation on each iteration to compute the projection onto the spectrahedron).
\item
We prove, that under Assumption \ref{ass:dist}, the above ``inefficient" algorithm, when initialized with a proper ``warm-start" vector,  guarantees that the projection onto the spectrahedron is always a rank-one matrix (hence,  only a rank-one SVD computation per iteration is required).
\item
Finally, we show that the nonconvex online gradient ascent algorithm, Algorithm \ref{alg:opm}, approximates sufficiently well the steps of the above algorithm (in case the projection is rank-one), avoiding SVD computations all together.
\end{enumerate}

We introduce the following notation that will be used throughout the analysis. For vectors in $\reals^d$ we let $\Vert{\cdot}\Vert$ denote the standard Euclidean norm, and for matrices in $\reals^{m\times n}$ we let $\Vert{\cdot}\Vert_F$ denote the Frobenius (Euclidean) norm and we let $\Vert{\cdot}\Vert_2$ denote the spectral norm (largest singular value). 
For all $t\in[T]$, we define $\X_t := \sum_{i=1}^{\ell}\x_t^{(i)}\x_t^{(i)\top}$, $\Q_t := \sum_{i=1}^{\ell}\q_t^{(i)}\q_t^{(i)\top}$, $\V_t := \sum_{i=1}^{\ell}\v_t^{(i)}\v_t^{(i)\top}$, and $\M_t := \sum_{i=1}^{\ell}\q_t^{(i)}\v_t^{(i)\top} + \v_t^{(i)}\q_t^{(i)\top}$. Note that $\X_t = \Q_t + \M_t + \V_t$. Recall that we let $\Q$ denote the covariance matrix associated with the distribution $\mD$ (as detailed in Assumption \ref{ass:dist}), and we let  $\lambda_1(\Q),\dots,\lambda_d(\Q)$ denote its eigenvalues in descending order. Also, we let $\x$ denote the leading eigenvector of $\Q$, which under Assumption \ref{ass:dist}, is unique.
We also define $\D_t := \Q_t-\ell\cdot\Q+\M_t$. Note that $\X_t = \ell\cdot\Q + \V_t + \D_t$. Intuitively, under Assumption \ref{ass:dist}, $\frac{1}{\ell}\D_t$ converges to zero in probability as $\ell\rightarrow\infty$.

We denote by $\mS$ the spectrahedron, i.e., $\mS := \{\W\in\reals^{d\times d} \,|\, \W\succeq 0, \trace(\W)=1\}$, and we let $\Pi_{\mS}[\W]$ denote the Euclidean projection of a symmetric matrix $\W\in\reals^{d\times d}$ onto $\mS$.

Our main building block towards proving Theorem \ref{thm:main} is to analyze the regret of a different non-convex algorithm for Online PCA. The meta-algorithm, Algorithm \ref{alg:1}, builds on the standard convexification scheme for Online PCA, i.e., ``lifting" the decision set from the unit ball to the spectrahedron, however, instead of computing exact projections onto the spectrahedron, it follows a nonconvex approach of only approximating the projection via a rank-one solution. We refer to it as a meta-algorithm, since for a given approximation parameter $\gamma$, it only requires on each iteration to find an approximate leading eigenvector of the matrix to be projected onto $\mS$. As in Algorithm \ref{alg:opm}, this algorithm also has two variants, one which does not include an additional regularizing term, and one which does.

Note that a straightforward implementation of Algorithm \ref{alg:1} with $\gamma = 0$ corresponds to updating $\hat{\w}_{t+1}$ via accurate SVD of the $d\times(\ell+1)$ matrix $(\sqrt{1-\eta_t\alpha}\hat{\w}_t, \sqrt{\eta_t}\x_t^{(1)},\dots,\sqrt{\eta_t}\x_t^{(\ell)})$, which already yields an algorithm with $O(\ell{}d)$ memory and $O(\ell{}d)$ amortized runtime per data-point.

Since we can always choose the regularization parameter $\alpha$ to be zero, in the sequel our analysis focuses only on the regularized variant of Algorithm \ref{alg:1}.

\begin{algorithm}
\caption{Approximate Non-convex Rank-one Online Gradient Ascent}
\label{alg:1}
\begin{algorithmic}[1]
\STATE input: unit vector $\hat{\w}_1$, sequence of positive learning rates $\{\eta_t\}_{t\geq 1}$, sequence of positive approximation parameters $\{\gamma_t\}_{t\geq 1}$, regularization parameter $\alpha \geq 0$
\FOR{$t=1\dots T$}
\STATE predict vector $\hat{\w}_t$
\STATE observe $\ell$ vectors $\x_t^{(1)},\dots,\x_t^{(\ell)}  $
\STATE $\hat{\w}_{t+1} \gets $ some unit vector satisfying: $\Vert{\hat{\w}_{t+1}\hat{\w}_{t+1}^{\top}-\w_{t+1}\w_{t+1}^{\top}}\Vert_F \leq \gamma_t$, where $\w_{t+1}$ is the leading eigenvector of either
\begin{eqnarray*}
& \W_{t+1}:=\hat{\w}_t\hat{\w}_t^{\top}+\eta_t\sum_{i=1}^{\ell}\x_t^{(i)}\x_t^{(i)\top} & \quad \{\textrm{without regularization}\} \\
& \textrm{OR} & \\
& \W_{t+1}:=(1-\eta_t\alpha)\hat{\w}_t\hat{\w}_t^{\top} + \eta_t\sum_{i=1}^{\ell}\x_t^{(i)}\x_t^{(i)\top} & \quad \{\textrm{with regularization}\}
\end{eqnarray*}

\ENDFOR
\end{algorithmic}
\end{algorithm}

\begin{lemma}\label{lem:projCond} 
Let $\w\in\reals^d$ be a unit vector and let $\X\in\reals^{d\times d}$ be positive semidefinite. Let $\w'$ be the leading eigenvector of the matrix $\W := (1-\eta\alpha)\w\w^{\top} + \eta\X$, for some $\eta > 0, \alpha\geq 0$ such that $\eta\alpha  < 1$.
If $\w^{\top}\X\w \geq \frac{\lambda_1(\X)+\lambda_2(\X)+\alpha}{2}$, then it follows that $\w'\w'^{\top} = \Pi_{\mS}[\W]$.
\end{lemma}

\begin{proof}
Recall $\w'$ denotes the leading eigenvector of $\W$ and let $\y_2,\dots,\y_d$ denote the other eigenvectors in non-increasing order (according to the corresponding eigenvalues). It is well known that the projection of $\W$ onto $\mS$ is given by
\begin{eqnarray*}
\Pi_{\mS}[\W] = (\lambda_1(\W)-\lambda)\w'\w'^{\top} + \sum_{i=2}^d\max\{0, \lambda_i(\W)-\lambda\}\y_i\y_i^{\top},
\end{eqnarray*}
where $\lambda$ is a non-negative real scalar such that $\lambda_1(\W)-\lambda + \sum_{i=2}^d\max\{\lambda_i(\W)-\lambda,0\} = 1$. Thus, if we show that $\lambda_1(\W) \geq 1 + \lambda_2(\W)$, then it in particular follows that $\Pi_{\mS}[\W] = \w'\w'^{\top}$.

Note that on one hand, 
\begin{eqnarray}\label{eq:projCond:1}
\lambda_1(\W) \geq \w^{\top}\W\w = 1-\eta\alpha + \eta\w^{\top}\X\w.
\end{eqnarray}

On the other-hand, using the last inequality, we can also write 
\begin{eqnarray}\label{eq:projCond:2}
\lambda_2(\W) &\leq &\lambda_1(\W) + \lambda_2(\W) - \w^{\top}\W\w \nonumber \\ 
&=& \lambda_1(\W)+\lambda_2(\W) -1 +\eta\alpha - \eta\w^{\top}\X\w.
\end{eqnarray}

Using Ky Fan's eigenvalue inequality, we have that
\begin{eqnarray}\label{eq:projCond:3}
 \lambda_1(\W)+\lambda_2(\W) &=& \lambda_1((1-\eta\alpha)\w\w^{\top}+\eta\X) + \lambda_2((1-\eta\alpha)\w\w^{\top}+\eta\X) \nonumber \\
 &\leq &\lambda_1((1-\eta\alpha)\w\w^{\top})+\lambda_2((1-\eta\alpha)\w\w^{\top}) + \lambda_1(\eta\X) + \lambda_2(\eta\X) \nonumber \\
 &=& 1-\eta\alpha + \eta\left({\lambda_1(\X)+\lambda_2(\X)}\right).
\end{eqnarray}
Thus, by combining Eq. \eqref{eq:projCond:1}, \eqref{eq:projCond:2}, \eqref{eq:projCond:3}, we arrive at the following sufficient condition so that $\w'\w'^{\top} = \Pi_{\mS}[\W]$:
\begin{eqnarray*}
1 - \eta\alpha + \eta\w^{\top}\X\w \geq 1 + \left({1 - \eta\alpha + \eta\left({\lambda_1(\X)+\lambda_2(\X)}\right)}\right) - \left({1 - \eta\alpha+ \eta\w^{\top}\X\w}\right),
\end{eqnarray*}
which is equivalent to the condition
$\w^{\top}\X\w \geq \frac{\lambda_1(\X)+\lambda_2(\X)}{2} + \frac{\alpha}{2}$.
\end{proof}

\begin{lemma}\label{lem:rank1ProjMat}
Suppose that on some iteration $t$ of Algorithm \ref{alg:1}  it holds that $\eta_t\alpha <1$ and
\begin{eqnarray*}
(\hat{\w}_t^{\top}\x)^2 \geq 1 - \frac{\delta(\Q)-V^2-\ell^{-1}\alpha}{2\lambda_1(\Q)} + \frac{2\Vert{\D_t}\Vert}{\ell\cdot\lambda_1(\Q)}
\end{eqnarray*}
Then, $\w_{t+1}\w_{t+1}^{\top} = \Pi_{\mS}[\W_{t+1}]$.
\end{lemma}

\begin{proof}
Recall that $\X_t := \sum_{i=1}^{\ell}\x_t^{(i)}\x_t^{(i)\top}$ and that $\W_{t+1} = (1-\eta_t\alpha)\hat{\w}_t\hat{\w}_t^{\top} + \eta_t\X_t$. 

Using Lemma \ref{lem:projCond} it suffices to show that
\begin{eqnarray}\label{eq:lem:rank1proj:0}
\hat{\w}_t^{\top}\X_t\hat{\w}_t \geq \frac{\lambda_1(\X_t)+\lambda_2(\X_t) + \alpha}{2}.
\end{eqnarray}

On one hand we have
\begin{eqnarray}\label{eq:lem:rank1proj:1}
\hat{\w}_t^{\top}\X_t\hat{\w}_t &=& \hat{\w}_t^{\top}(\ell\cdot\Q + \V_t + \D_t)\hat{\w}_t \geq \ell\cdot\hat{\w}_t^{\top}\Q\hat{\w}_t - \Vert{\D_t}\Vert \nonumber \\
&\geq &\ell(\hat{\w}_t^{\top}\x)^2\cdot\x^{\top}\Q\x - \Vert{\D_t}\Vert = \ell(\hat{\w}_t^{\top}\x)^2\cdot\lambda_1(\Q) - \Vert{\D_t}\Vert.
\end{eqnarray}

On the other hand,
\begin{eqnarray}\label{eq:lem:rank1proj:2}
\lambda_1(\X_t)+\lambda_2(\X_t) &=& \lambda_1(\ell\cdot\Q + \V_t + \D_t) + \lambda_2(\ell\cdot\Q + \V_t + \D_t) \nonumber \\
&\underset{(a)}{\leq} & \lambda_1(\ell\cdot\Q+\V_t) + \lambda_2(\ell\cdot\Q+\V_t) + 2\Vert{\D_t}\Vert \nonumber \\
&\underset{(b)}{\leq} & \lambda_1(\ell\cdot\Q)  + \lambda_2(\ell\cdot\Q) + \lambda_1(\V_t) + \lambda_2(\V_t) + 2\Vert{\D_t}\Vert \nonumber \\
&\leq & \ell(\lambda_1(\Q)  + \lambda_2(\Q)) + \trace(\V_t) + 2\Vert{\D_t}\Vert \nonumber \\
&\leq & \ell(\lambda_1(\Q)  + \lambda_2(\Q) + V^2) + 2\Vert{\D_t}\Vert,
\end{eqnarray}
where (a) follows from Weyl's eigenvalue inequality, and (b) follows from Ky Fan's eigenvalue inequality.

Combining Eq. \eqref{eq:lem:rank1proj:0}, \eqref{eq:lem:rank1proj:1}, \eqref{eq:lem:rank1proj:2}, we arrive at the following sufficient condition so that $\w_{t+1}\w_{t+1}^{\top} = \Pi_{\mS}[\W_{t+1}]$:
\begin{eqnarray*}
(\hat{\w}_t^{\top}\x)^2 &\geq& \frac{\lambda_1(\Q) + \lambda_2(\Q) + V^2 + 4\ell^{-1}\Vert{\D_t}\Vert + \ell^{-1}\alpha}{2\cdot\lambda_1(\Q)}   \\
&=& \frac{2\lambda_1(\Q) - \delta(\Q) + V^2 + 4\ell^{-1}\Vert{\D_t}\Vert + \ell^{-1}\alpha}{2\lambda_1(\Q)} \\
&=& 1 - \frac{\delta(\Q)-V^2-\ell^{-1}\alpha}{2\lambda_1(\Q)}  +\frac{2\Vert{\D_t}\Vert}{\ell\cdot\lambda_1(\Q)}.
\end{eqnarray*}

\end{proof}

\begin{lemma}\label{lem:corrIncBlock}
Suppose that on some iteration $t$ of Algorithm \ref{alg:1} it holds that $(\x^{\top}\hat{\w}_t)^2 \geq \frac{1}{2}$. Then, for any learning rate $\eta_t >0$ and $\alpha>0$ such that $\eta_t\alpha < 1$, it holds that
\begin{eqnarray*}
(\x^{\top}\w_{t+1})^2 \geq (\x^{\top}\hat{\w}_t)^2 + \eta_t\ell\frac{(1-(\x^{\top}\hat{\w}_t)^2)\cdot\delta(\Q) - (\x^{\top}\hat{\w}_t)^2V^2-4\ell^{-1}\Vert{\D_t}\Vert}{\lambda_1(\W_{t+1})-\lambda_2(\W_{t+1})}.
\end{eqnarray*}
\end{lemma}
\begin{proof}
Fix some iteration $t$. We introduce the short notation $\w = \hat{\w}_t$, $\w'=\w_{t+1}$, $\W = \W_{t+1} = (1-\eta_t\alpha)\hat{\w}_t\hat{\w}_t^{\top} + \eta_t\X_t$, $\lambda_1 = \lambda_1(\W_{t+1})$, $\lambda_2 = \lambda_2(\W_{t+1})$, and for all $i\geq 2$, $\y_i$ is the eigenvector of $\W_{t+1}$ associated with eigenvalue $\lambda_i = \lambda_i(\W_{t+1})$.

It holds that
\begin{eqnarray*}
\lambda_1\w'\w' + \sum_{i=2}^d\lambda_i\y_i\y_i^{\top} = \W = (1-\eta_t\alpha)\w\w^{\top}+\eta_t\X_t.
\end{eqnarray*}

Thus, we have that
\begin{eqnarray*}
(\x^{\top}\w')^2 &=& \frac{\x^{\top}\W\x - \sum_{i=2}^d\lambda_i(\x^{\top}\y_i)^2}{\lambda_1} \\
&=& \frac{(1-\eta_t\alpha)(\x^{\top}\w)^2 + \eta_t\x^{\top}\X_t\x - \sum_{i=2}^d\lambda_i(\x^{\top}\y_i)^2}{\lambda_1}.
\end{eqnarray*}

Note that $\sum_{i=2}^d\lambda_i(\x^{\top}\y_i)^2 \leq \lambda_2(\Vert{\x}\Vert^2 - (\x^{\top}\w')^2) = \lambda_2(1 - (\x^{\top}\w')^2)$. Thus, we have that

\begin{eqnarray*}
(\x^{\top}\w')^2  \geq \frac{(1-\eta_t\alpha)(\x^{\top}\w)^2 + \eta_t\x^{\top}\X_t\x - \lambda_2(1 - (\x^{\top}\w')^2)}{\lambda_1}.
\end{eqnarray*}

Rearranging we obtain,
\begin{eqnarray*}
(\x^{\top}\w')^2 &\geq &\frac{(1-\eta_t\alpha)(\x^{\top}\w)^2 +  \eta_t\x^{\top}\X_t\x - \lambda_2}{\lambda_1-\lambda_2} \\
&=& (\x^{\top}\w)^2 + \frac{\eta_t\x^{\top}\X_t\x + (1-\eta_t\alpha-\lambda_1+\lambda_2)(\x^{\top}\w)^2 - \lambda_2}{\lambda_1-\lambda_2} \\
&= & (\x^{\top}\w)^2 + \frac{\eta_t\x^{\top}\X_t\x + (1-\eta_t\alpha-(\lambda_1+\lambda_2))(\x^{\top}\w)^2 + \lambda_2(2(\x^{\top}\w)^2-1)}{\lambda_1-\lambda_2}.
\end{eqnarray*}

Note that via Ky Fan's inequality we have that
\begin{eqnarray*}
\lambda_1 + \lambda_2 &\leq &\lambda_1((1-\eta_t\alpha)\w\w^{\top}) + \lambda_2((1-\eta_t\alpha)\w\w^{\top}) + \lambda_1(\eta_t\X_t) + \lambda_2(\eta_t\X_t) \\
&=& 1 - \eta_t\alpha + \eta_t(\lambda_1(\X_t) + \lambda_2(\X_t)).
\end{eqnarray*}

Also, $\lambda_2 = \lambda_2((1-\eta_t\alpha)\w\w^{\top}+\eta_t\X_t) \geq \lambda_2(\eta_t\X_t) = \eta_t\lambda_2(\X_t)$.

Thus, using our assumption that $(\x^{\top}\w)^2 \geq 1/2$, we have that
\begin{eqnarray*}
(\x^{\top}\w')^2 \geq (\x^{\top}\w)^2 + \eta_t\frac{\x^{\top}\X_t\x - (\lambda_1(\X_t) + \lambda_2(\X_t))(\x^{\top}\w)^2 + \lambda_2(\X_t)(2(\x^{\top}\w)^2-1)}{\lambda_1-\lambda_2}. 
\end{eqnarray*}

Note that
\begin{eqnarray*}
\x^{\top}\X_t\x = \x^{\top}(\ell\cdot\Q + \V_t + \D_t)\x \geq \x^{\top}(\ell\cdot\Q + \D_t)\x \geq \ell\cdot\lambda_1(\Q) - \Vert{\D_t}\Vert,
\end{eqnarray*}
\begin{eqnarray*}
\lambda_2(\X_t) = \lambda_2(\ell\cdot\Q + \V_t + \D_t) \geq \lambda_2(\ell\cdot\Q + \D_t) \geq \ell\cdot\lambda_2(\Q) - \Vert{\D_t}\Vert,
\end{eqnarray*}
\begin{eqnarray*}
\lambda_1(\X_t) + \lambda_2(\X_t) &\underset{(a)}{\leq} &\lambda_1(\ell\cdot\Q) + \lambda_2(\ell\cdot\Q) + \lambda_1(\V_t+\D_t) + \lambda_2(\V_t+\D_t) \\
&\leq & \ell(\lambda_1(\Q) + \lambda_2(\Q)) + \lambda_1(\V_t) + \lambda_2(\V_t) + 2\Vert{\D_t}\Vert \\
&\leq & \ell(\lambda_1(\Q) + \lambda_2(\Q)) + \trace(\V_t) + 2\Vert{\D_t}\Vert \\
&\leq & \ell(\lambda_1(\Q) + \lambda_2(\Q) + V^2)  + 2\Vert{\D_t}\Vert,
\end{eqnarray*}
where (a) follows again from Ky Fan's inequality.

Plugging-in all of the above bounds, we have that
\begin{eqnarray*}
(\x^{\top}\w')^2
&\geq&
(\x^{\top}\w)^2 \\
&&+ \eta_t\ell\frac{\lambda_1(\Q) - (\lambda_1(\Q)+\lambda_2(\Q) + V^2)\cdot(\x^{\top}\w)^2 + \lambda_2(\Q)\cdot(2(\x^{\top}\w)^2-1) - 4\ell^{-1}\Vert{\D_t}\Vert}{\lambda_1-\lambda_2} \\
&=&
(\x^{\top}\w)^2 + \eta_t\ell\frac{(1-(\w^{\top}\x)^2)\cdot(\lambda_1(\Q)-\lambda_2(\Q)) - (\x^{\top}\w)^2{}V^2 - 4\ell^{-1}\Vert{\D_t}\Vert}{\lambda_1-\lambda_2} \\
&=&(\x^{\top}\w)^2 + \eta_t\ell\frac{(1-(\x^{\top}\w)^2)\cdot\delta(\Q) - (\x^{\top}\w)^2V^2-4\ell^{-1}\Vert{\D_t}\Vert}{\lambda_1-\lambda_2}.
\end{eqnarray*}
\end{proof}

\begin{lemma}\label{lem:goodProj}
Suppose that when applying Algorithm \ref{alg:1}, the following conditions hold:
\begin{align}
\forall t\in[T]:~\frac{1}{\ell}\Vert{\D_t}\Vert & \leq  \epsilon \leq \frac{1}{72\lambda_1(\Q)}\left({\delta(\Q)^2-V^4-2V^2\lambda_1(\Q)}\right) \label{eq:epsilonBound},\\
\eta_t &\leq  \min\Big\{\frac{\epsilon}{4\ell\lambda_1(\Q)\cdot(V^2+4\epsilon)},~\frac{1}{\ell(R+V)^2}\Big\}, \nonumber \\
\gamma_t &\leq \min\Big\{\frac{\epsilon}{4\lambda_1(\Q)}, ~18\epsilon\eta_t\ell\Big\}, ~ \textrm{and} \nonumber \\
\alpha &\leq  \frac{\ell}{4(\delta(\Q)+V^2)}\left({\delta(\Q)^2-V^4-2V^2\lambda_1(\Q)}\right), \nonumber \\
(\w_1^{\top}\x)^2 & \geq 1 - \frac{\delta(\Q)-V^2-\ell^{-1}\alpha -4\epsilon}{2\lambda_1(\Q)}. \nonumber
\end{align}
Then, for all $t\in[T]$, $\w_{t+1}\w_{t+1}^{\top} = \Pi_{\mS}[\W_{t+1}]$.
\end{lemma}

\begin{proof}
Note that under the assumptions of the lemma it holds on any iteration $t$ that $\eta_t\alpha <1$. Thus, in light of Lemma \ref{lem:rank1ProjMat}, it suffices to show that on each iteration $t$, it holds that
\begin{eqnarray*}
(\hat{\w}_t^{\top}\x)^2 \geq 1 - \frac{\delta(\Q)-V^2 -\ell^{-1}\alpha -4\epsilon}{2\lambda_1(\Q)}.
\end{eqnarray*}

We prove this inequality indeed holds for all $t\in[T]$ by induction. Note that for $t=1$, this clearly holds by our assumption on $\hat{\w}_1$.

Suppose now the assumption holds for some $t \geq 1$. In the following we let $\lambda_i$ denote the $i$th largest eigenvalue of the matrix $\W_{t+1}:=(1-\eta_t\alpha)\hat{\w}_t\hat{\w}_t^{\top}+\eta_t\X_t$. Note that under the induction hypothesis and Assumption \ref{ass:dist}, it in particular holds that $(\hat{\w}_t^{\top}\x)^2 \geq 1/2$, and hence we can invoke Lemma \ref{lem:corrIncBlock}.

We consider two cases. If $(\hat{\w}_t^{\top}\x)^2 \geq1 - \frac{\delta(\Q)-V^2-\ell^{-1}\alpha-5\epsilon}{2\lambda_1(\Q)}$, then using Lemma \ref{lem:corrIncBlock} we have that
\begin{eqnarray*}
(\w_{t+1}^{\top}\x)^2 &\geq &1 - \frac{\delta(\Q)-V^2 -\ell^{-1}\alpha -5\epsilon}{2\lambda_1(\Q)} - \eta_t\ell\frac{V^2+4\epsilon}{\lambda_1-\lambda_2} \\
&\underset{(a)}{\geq} &1 - \frac{\delta(\Q)-V^2 -\ell^{-1}\alpha -5\epsilon}{2\lambda_1(\Q)} - \eta_t\ell\left({V^2+4\epsilon}\right),
\end{eqnarray*}
were (a) follows since under the induction hypothesis, we in particular have that  $\w_{t+1}\w_{t+1}^{\top}=\Pi_{\mS}[\W_{t+1}]$ (see Lemma \ref{lem:rank1ProjMat}), which in turn implies that $\lambda_1 \geq 1 + \lambda_2$ (see proof of Lemma \ref{lem:projCond}).

Thus, for any $\eta_t \leq \frac{\epsilon}{4\ell\lambda_1(\Q)\cdot(V^2+4\epsilon)}$ we obtain
\begin{eqnarray*}
(\w_{t+1}^{\top}\x)^2 \geq1 - \frac{\delta(\Q)-V^2 -\ell^{-1}\alpha -\frac{9}{2}\epsilon}{2\lambda_1(\Q)}.
\end{eqnarray*}
Moreover, we have that
\begin{eqnarray*}
(\hat{\w}_{t+1}^{\top}\x)^2 &\geq& (\w_{t+1}^{\top}\x)^2 - \Vert{\w_{t+1}\w_{t+1}^{\top}-\hat{\w}_{t+1}\hat{\w}_{t+1}}\Vert_F \geq (\w_{t+1}^{\top}\x)^2 - \gamma_t \\
&\geq & 1 - \frac{\delta(\Q)-V^2 -\ell^{-1}\alpha -\frac{9}{2}\epsilon}{2\lambda_1(\Q)} - \gamma_t.
\end{eqnarray*}
Thus, for any $\gamma_t \leq \frac{\epsilon}{4\lambda_1(\Q)}$ the claim indeed holds for the first case.

On the other hand, in case $(\hat{\w}_t^{\top}\x)^2 < 1 - \frac{\delta(\Q)-V^2 -\ell^{-1}\alpha -5\epsilon}{2\lambda_1(\Q)} $, by an application of Lemma \ref{lem:corrIncBlock} we have that
\begin{eqnarray*}
(\w_{t+1}^{\top}\x)^2 &\geq &(\hat{\w}_t^{\top}\x)^2 + \eta_t\ell\frac{(1-(\hat{\w}_t^{\top}\x)^2)\cdot\delta(\Q) - (\hat{\w}_t^{\top}\x)^2V^2-4\epsilon}{\lambda_1-\lambda_2} \\
&\underset{(a)}{\geq} & (\hat{\w}_t^{\top}\x)^2 + \eta_t\ell\frac{\delta(\Q) - (\delta(\Q)+V^2)\left({1 - \frac{\delta(\Q)-V^2 -\ell^{-1}\alpha -5\epsilon}{2\lambda_1(\Q)}}\right) - 4\epsilon}{\lambda_1-\lambda_2} \\
&= & (\hat{\w}_t^{\top}\x)^2 + \eta_t\ell\frac{\frac{\delta(\Q)-V^2-\ell^{-1}\alpha-5\epsilon}{2\lambda_1(\Q)}\cdot\delta(\Q) - V^2\cdot\left({1 - \frac{\delta(\Q)-V^2-\ell^{-1}\alpha-5\epsilon}{2\lambda_1(\Q)}}\right) - 4\epsilon}{\lambda_1-\lambda_2} \\
&\underset{(b)}{\geq} & (\hat{\w}_t^{\top}\x)^2 + \eta_t\ell\frac{\frac{\delta(\Q)-V^2-\ell^{-1}\alpha}{2\lambda_1(\Q)}\cdot\delta(\Q) - V^2\cdot\left({1 - \frac{\delta(\Q)-V^2-\ell^{-1}\alpha}{2\lambda_1(\Q)}}\right) - 9\epsilon}{\lambda_1-\lambda_2} \\
&= & (\hat{\w}_t^{\top}\x)^2 + \eta_t\ell\frac{\delta(\Q)^2-V^4-2V^2\lambda_1(\Q) -  (\delta(\Q)+V^2)\ell^{-1}\alpha- 18\epsilon\lambda_1(\Q)}{2\lambda_1(\Q)(\lambda_1-\lambda_2)},
\end{eqnarray*}
where (a) follows from our assumption on $(\hat{\w}_t^{\top}\x)^2$ in this second case, and (b) follows, since  Assumption \ref{ass:dist} implies that $\max\{\delta(\Q),V^2\} \leq \lambda_1(\Q)$.

Thus, for any
\begin{eqnarray*}
\epsilon &\leq & \frac{1}{72\lambda_1(\Q)}\left({\delta(\Q)^2-V^4-2V^2\lambda_1(\Q)}\right) \quad \textrm{and} \\
\alpha &\leq & \frac{\ell}{4(\delta(\Q)+V^2)}\left({\delta(\Q)^2-V^4-2V^2\lambda_1(\Q)}\right)
\end{eqnarray*}
we have that
\begin{eqnarray*}
(\w_{t+1}^{\top}\x)^2 &\geq &(\hat{\w}_t^{\top}\x)^2 + \eta_t\ell\frac{\delta(\Q)^2-V^4-2V^2\lambda_1(\Q)}{4\lambda_1(\Q)(\lambda_1-\lambda_2)}.
\end{eqnarray*}

Moreover, as before, we have that
\begin{eqnarray*}
(\hat{\w}_{t+1}^{\top}\x)^2 &\geq& (\w_{t+1}^{\top}\x)^2 - \gamma_t \\
& \geq &(\hat{\w}_t^{\top}\x)^2  + \eta_t\ell\frac{\delta(\Q)^2-V^4-2V^2\lambda_1(\Q)}{4\lambda_1(\Q)(\lambda_1-\lambda_2)} - \gamma_t \\
&\underset{(a)}{\geq} &(\hat{\w}_t^{\top}\x)^2  + \eta_t\ell\frac{\delta(\Q)^2-V^4-2V^2\lambda_1(\Q)}{4\lambda_1(\Q)} - \gamma_t,
\end{eqnarray*}
where (a) follows since Assumption \ref{ass:dist} implies that $\delta(\Q)^2-V^4-2V^2\lambda_1(\Q) \geq 0$, and since
\begin{eqnarray*}
\lambda_1-\lambda_2 \leq \lambda_1 \leq 1-\eta_t\alpha + \eta_t\Vert{\X_t}\Vert \leq 1 + \eta_t\sum_{i=1}^{\ell}\Vert{\q_t^{(i)}+\v_t^{(i)}}\Vert^2 \leq  1+ \eta_t\ell(R+V)^2 \leq 2,
\end{eqnarray*}
where the last inequality follows from our assumption that $\eta_t \leq \frac{1}{\ell(R+V)^2}$.
Thus, for any $\gamma_t \leq \eta_t\ell\frac{\delta(\Q)^2-V^4-2V^2\lambda_1(\Q)}{4\lambda_1(\Q)}$ (which in particular holds for $\gamma_t \leq 18\eta_t\ell\epsilon$), we have that 
\begin{eqnarray*}
(\hat{\w}_{t+1}^{\top}\x)^2 \geq (\hat{\w}_t^{\top}\x)^2 \geq 1 - \frac{\delta(\Q)-V^2-\ell^{-1}\alpha-4\epsilon}{2\lambda_1(\Q)},
\end{eqnarray*}
as needed.
\end{proof}

\begin{lemma}[Convergence of Algorithm \ref{alg:1}]\label{lem:metaAlgConv}
Consider applying Algorithm \ref{alg:1}  to a sequence of vectors $\{(\x_t^{(1)},\dots,\x_t^{(\ell)})\}_{t\in[T]}$ which follow Assumption \ref{ass:dist}, and suppose that all conditions stated in Lemma \ref{lem:goodProj} hold. Then, for any unit vector $\w$ it holds that
\begin{align*}
&\sum_{t=1}^T\sum_{i=1}^{\ell}(\w^{\top}\x_t^{(i)})^2 - \sum_{t=1}^T\sum_{i=1}^{\ell}(\hat{\w}_t^{\top}\x_t^{(i)})^2  \leq \\
& \frac{1}{\eta_1} +\sum_{t=2}^T\left({\frac{1-\eta_t\alpha}{2\eta_t} - \frac{1}{2\eta_{t-1}}}\right)\Vert{\hat{\w}_t\hat{\w}_t^{\top} - \w\w^{\top}}\Vert_F^2 \\
&+\sum_{t=1}^T\left({\frac{3\sqrt{2}}{2}\frac{\gamma_t}{\eta_t} + \frac{\eta_t}{2}( \ell^2(R+V)^4+\alpha^2)}\right).
\end{align*}
\end{lemma}

\begin{proof}
Fix some unit vector $\w$. By an application of Lemma \ref{lem:goodProj}, it holds for all $t\in[T]$ that $\w_{t+1}\w_{t+1}^{\top} = \Pi_{\mS}[\W_{t+1}]$.

Thus, using standard arguments\footnote{see for instance the analysis of Online Gradient Descent in \cite{Hazan16}}, we have that for all $t\in[T]$ it holds that
\begin{eqnarray*}
\Vert{\w_{t+1}\w_{t+1}^{\top}-\w\w^{\top}}\Vert_F^2 &\leq &\Vert{\W_{t+1}-\w\w^{\top}}\Vert_F^2 \\
&= &\Vert{(1-\eta_t\alpha)\hat{\w}_t\hat{\w}_t^{\top}+\eta_t\X_t-\w\w^{\top}}\Vert_F^2 \\
&=&
\Vert{\hat{\w}_t\hat{\w}_t^{\top} - \w\w^{\top}}\Vert_F^2 \\
&&+ 2\eta_t(\hat{\w}_t\hat{\w}_t^{\top} - \w\w^{\top})\bullet(\X_t-\alpha\hat{\w}_t\hat{\w}_t^{\top}) \\
&&+ \eta_t^2\Vert{\X_t-\alpha\hat{\w}_t\hat{\w}_t^{\top}}\Vert_F^2.
\end{eqnarray*}

Note that
\begin{align*}
\Vert{\hat{\w}_{t+1}\hat{\w}_{t+1}^{\top} - \w\w^{\top}}\Vert_F^2  &= \Vert{\hat{\w}_{t+1}\hat{\w}_{t+1}^{\top} + \w_{t+1}\w_{t+1}^{\top} - \w_{t+1}\w_{t+1}^{\top} - \w\w^{\top}}\Vert_F^2 \\
&\underset{(a)}{\leq}  \Vert{\w_{t+1}\w_{t+1}^{\top} - \w\w^{\top}}\Vert_F^2 + 3\sqrt{2}\Vert{\w_{t+1}\w_{t+1}^{\top} - \hat{\w}_{t+1}\hat{\w}_{t+1}^{\top}}\Vert_F \\
&\leq \Vert{\w_{t+1}\w_{t+1}^{\top} - \w\w^{\top}}\Vert_F^2 + 3\sqrt{2}\gamma_t,
\end{align*}
where (a) follows since for any two unit vectors $\y,\z$ it holds that $\Vert{\y\y^{\top}-\z\z^{\top}}\Vert_F \leq \sqrt{2}$.

Note also that
\begin{eqnarray*}
2(\hat{\w}_t\hat{\w}_t^{\top} - \w\w^{\top})\bullet(-\alpha\hat{\w}_t\hat{\w}_t^{\top}) = -\alpha(2 - 2(\w^{\top}\hat{\w}_t)^2) = -\alpha\Vert{\hat{\w}_t\hat{\w}_t^{\top}-\w\w^{\top}}\Vert_F^2.
\end{eqnarray*}

Combining the three bounds above, we obtain
\begin{align*}
&(\w\w^{\top}-\hat{\w}_t\hat{\w}_t^{\top})\bullet\X_t \leq \\
&\frac{1}{2\eta_t}\left({(1-\eta_t\alpha)\Vert{\hat{\w}_t\hat{\w}_t^{\top} - \w\w^{\top}}\Vert_F^2 -  \Vert{\hat{\w}_{t+1}\hat{\w}_{t+1}^{\top} - \w\w^{\top}}\Vert_F^2}\right) \\
&+ \frac{3\sqrt{2}\gamma_t}{2\eta_t} + \frac{\eta_t}{2}\Vert{\X_t-\alpha\hat{\w}_t\hat{\w}_t^{\top}}\Vert_F^2.
\end{align*}

Note that $\Vert{\X_t-\alpha\hat{\w}_t\hat{\w}_t^{\top}}\Vert_F^2 \leq \Vert{\X_t}\Vert_F^2 + \alpha^2$. Summing over all iterations we obtain the bound
\begin{eqnarray*}
\sum_{t=1}^T\w^{\top}\X_t\w - \sum_{t=1}^T\hat{\w}_t^{\top}\X_t\hat{\w}_t &\leq &\frac{1-\eta_1\alpha}{2\eta_1}\Vert{\hat{\w}_1\hat{\w}_1^{\top} - \w\w^{\top}}\Vert_F^2 \\
&&+\sum_{t=2}^T\left({\frac{1-\eta_t\alpha}{2\eta_t} - \frac{1}{2\eta_{t-1}}}\right)\Vert{\hat{\w}_t\hat{\w}_t^{\top} - \w\w^{\top}}\Vert_F^2 \\
&&+\sum_{t=1}^T\left({\frac{3\sqrt{2}}{2}\frac{\gamma_t}{\eta_t} + \frac{\eta_t}{2}(\Vert{\X_t}\Vert_F^2+\alpha^2)}\right).
\end{eqnarray*}

Finally, note that $\Vert{\hat{\w}_1\hat{\w}_1^{\top} - \w\w^{\top}}\Vert_F^2 \leq 2$ and that under Assumption \ref{ass:dist}, it holds for all $t\in[T]$ that
\begin{eqnarray}\label{eq:normBound}
\Vert{\X_t}\Vert_F &=& \Vert{\sum_{i=1}^{\ell}\x_t^{(i)}\x_t^{(i)\top}}\Vert_F \leq \sum_{i=1}^{\ell}\Vert{\x_t^{(i)}\x_t^{(i)\top}}\Vert_F =\sum_{i=1}^{\ell}\Vert{\x_t^{(i)}}\Vert^2 \nonumber \\
&=& \sum_{i=1}^{\ell}\Vert{\q_t^{(i)}+\v_t^{(i)}}\Vert^2 \leq \ell(R+V)^2.
\end{eqnarray}
Hence, the lemma follows.
\end{proof}

As an example, if we invoke Lemma \ref{lem:metaAlgConv} with a fixed learning rate $\eta>0$ (i.e, $\eta_t=\eta$ for all $t\in[T]$), zero regularization ($\alpha =0$), and $\gamma_t=0$ for all $t\in[T]$, we get a regret bound 
\begin{eqnarray*}
\frac{1}{\eta} + \frac{\eta}{2}{}T\ell^2(R+V)^4,
\end{eqnarray*}
 which by an appropriate choice of $\eta$ and treating $\ell$ as a constant, yields the familiar $O(\sqrt{T})$ regret bound for Online Gradient Ascent with linear payoff functions \cite{Hazan16}. A Similar treatment with $\alpha >0$ and vanishing step-size can be shown (as we indeed show in the sequel) to yield a $O(\log{}T)$ regret bound.

 \subsection{Convergence of Algorithm \ref{alg:opm}}
 
\begin{lemma}\label{lem:opmApprox}
Consider some iteration $t$ of Algorithm \ref{alg:opm}, and let $\w_{t+1}$ denote the leading eigenvector of the matrix $\W_{t+1}:=(1+\eta_t\alpha)\hat{\w}_t\hat{\w}_t^{\top}+\eta_t\X_t$. If $\eta_t \leq \frac{1}{3\ell(R+V)^2}$ and $\alpha \leq \ell(R+V)^2$, then it holds that 
\begin{eqnarray*}
\Vert{\hat{\w}_{t+1}\hat{\w}_{t+1} - \w_{t+1}\w_{t+1}^{\top}}\Vert_F \leq \sqrt{33}(\eta_t\ell(R+V)^2)^2.
\end{eqnarray*}

\end{lemma}
\begin{proof}
Let us denote by $\y_2,\dots,\y_d$ the $(d-1)$ non-leading eigenvectors of the matrix $\W_{t+1}$. 
Since both $\w_{t+1}, \hat{\w}_{t+1}$ are unit vectors, we have that
\begin{eqnarray}\label{eq:opmApprox:1}
\Vert{\hat{\w}_{t+1}\hat{\w}_{t+1} -\w_{t+1}\w_{t+1}^{\top}}\Vert_F^2 = 2\left({1 - (\hat{\w}_{t+1}^{\top}\w_{t+1})^2}\right) = 2\sum_{i=2}^d(\hat{\w}_{t+1}^{\top}\y_i)^2.
\end{eqnarray}

Note that by the update rule of Algorithm \ref{alg:opm} and since $\hat{\w}_t$ is a unit vector, the vector $\hat{\w}_{t+1}$ could be written as
\begin{eqnarray*}
\hat{\w}_{t+1}&=& \frac{\hat{\w}_t + \eta_t(\X_t\hat{\w}_t-\alpha\hat{\w}_t)}{\Vert{\hat{\w}_t + \eta_t(\X_t\hat{\w}_t-\alpha\hat{\w}_t)}\Vert} \\
&=&
\frac{\left({\hat{\w}_t\hat{\w}_t^{\top} + \eta_t(\X_t-\alpha\hat{\w}_t\hat{\w}_t^{\top})}\right)\hat{\w}_t}{\Vert{\left({\hat{\w}_t\hat{\w}_t^{\top} + \eta_t(\X_t-\alpha\hat{\w}_t\hat{\w}_t^{\top})}\right)\hat{\w}_t}\Vert}  =
\frac{\W_{t+1}\hat{\w}_t}{\Vert{\W_{t+1}\hat{\w}_t}\Vert}.
\end{eqnarray*}
Hence, $\hat{\w}_{t+1}$, is the result of applying a single iteration of the well known Power Method for leading eigenvector computation, initialized with the vector $\hat{\w}_t$, to the matrix $\W_{t+1}$. Let us denote by $\y_2,\dots,\y_d$ the $(d-1)$ non-leading eigenvectors of $\W_{t+1}$. Using standard arguments, see for instance Eq. (18) in \cite{G15}, we have that

\begin{eqnarray*}
\sum_{i=2}^d(\hat{\w}_{t+1}^{\top}\y_i)^2 \leq \frac{\sum_{i=2}^d(\hat{\w}_t^{\top}\y_i)^2}{(\hat{\w}_t^{\top}\w_{t+1})^2}\left({\frac{\lambda_2(\W_{t+1})}{\lambda_1(\W_{t+1})}}\right)^2.
\end{eqnarray*}

Since $\hat{\w}_{t}$ is a unit vector, we have that $\sum_{i=2}^d(\hat{\w}_{t}^{\top}\y_i)^2 = 1 - (\hat{\w}_{t}^{\top}\w_{t+1})^2$. Moreover, we can bound 
\begin{eqnarray*}
\frac{\lambda_2(\W_{t+1})}{\lambda_1(\W_{t+1})} &=& \frac{\lambda_2((1-\eta_t\alpha)\hat{\w}_t\hat{\w}_t^{\top}+\eta_t\X_t)}{\lambda_1((1-\eta_t\alpha)\hat{\w}_t\hat{\w}_t^{\top}+\eta_t\X_t)}\\
&\leq & \frac{\lambda_2((1-\eta_t\alpha)\hat{\w}_t\hat{\w}_t^{\top}) + \lambda_1(\eta_t\X_t)}{\lambda_1((1-\eta_t\alpha)\hat{\w}_t\hat{\w}_t^{\top})} 
\leq \frac{\eta_t\Vert{\X_t}\Vert}{1-\eta_t\alpha},
\end{eqnarray*}
where the first inequality follows from Weyl's inequality for the eigenvalues. Thus, plugging-in into Eq. \eqref{eq:opmApprox:1}, we have that

\begin{eqnarray}\label{eq:opmApprox:2}
\Vert{\hat{\w}_{t+1}\hat{\w}_{t+1} -\w_{t+1}\w_{t+1}^{\top}}\Vert_F^2 \leq 2\frac{1 - (\hat{\w}_t^{\top}\w_{t+1})^2}{(\hat{\w}_t^{\top}\w_{t+1})^2}\frac{\eta_t^2\Vert{\X_t}\Vert^2}{(1-\eta_t\alpha)^2}.
\end{eqnarray}

Using the Davis-Kahan sin$\theta$ theorem (see for instance Theorem 4 in \cite{G15online}), we have that
\begin{eqnarray*}
1 - (\hat{\w}_t^{\top}\w_{t+1})^2 &=& \frac{1}{2}\Vert{\hat{\w}_{t}\hat{\w}_{t}^{\top} -\w_{t+1}\w_{t+1}^{\top}}\Vert_F^2 \\
&\leq &\frac{4\Vert{\hat{\w}_{t}\hat{\w}_{t}^{\top}-\W_{t+1}}\Vert^2}{(\lambda_1(\hat{\w}_t\hat{\w}_t^{\top}) - \lambda_2(\hat{\w}_t\hat{\w}_t^{\top}))^2} \\
&=& \frac{4\Vert{\eta_t\alpha\hat{\w}_t\hat{\w}_t^{\top}-\eta_t\X_t}\Vert^2}{\lambda_1(\hat{\w}_t\hat{\w}_t^{\top})^2}\leq 
4\eta_t^2\max\{\Vert{\X_t}\Vert^2,\alpha^2\}.
\end{eqnarray*}

Plugging back into Eq. \eqref{eq:opmApprox:2} and using the fact that $\eta_t \leq \frac{1}{3\max\{\Vert{\X_t}\Vert,\alpha\}}$ (see bound on $\Vert{\X_t}\Vert$ in Eq. \eqref{eq:normBound}), we can conclude that
\begin{eqnarray*}
\Vert{\hat{\w}_{t+1}\hat{\w}_{t+1} -\w_{t+1}\w_{t+1}^{\top}}\Vert_F^2 &\leq& \frac{8\eta_t^4\Vert{\X_t}\Vert^2\max\{\Vert{\X_t}\Vert^2,\alpha^2\}}{(1-4\eta_t^2\max\{\Vert{\X_t}\Vert^2,\alpha^2\})(1-\eta_t\alpha)^2}\\
& {\leq} &\frac{8\eta_t^4\Vert{\X_t}\Vert^2\max\{\Vert{\X_t}\Vert^2,\alpha^2\}}{\left({1-\frac{4}{9}}\right)\left({1-\frac{1}{3}}\right)^2}\\
&{\leq} &33(\eta_t\ell(R+V)^2)^4,
\end{eqnarray*}
where all inequalities follow from our assumptions on $\eta_t,\alpha$ and the bound \eqref{eq:normBound}.
\end{proof}

\begin{lemma}[Matrix Hoeffding]\label{lem:matHoff}
Under the conditions of Assumption \ref{ass:dist}, it holds for all $t\in[T]$ and for all $\epsilon > 0$ that
\begin{eqnarray*}
\Pr\left({\Vert{\frac{1}{\ell}\D_t}\Vert\geq \epsilon}\right) \leq 2d\cdot\exp\left({-\frac{\epsilon^2\ell}{128R^4}}\right).
\end{eqnarray*}
\end{lemma}
\begin{proof}
By a straightforward application of the Matrix Hoeffding inequality (see for instance \cite{Tropp12}), we have for any fixed $t\in[T]$ that
\begin{eqnarray*}
&\Pr\left({\Vert{\frac{1}{\ell}(\Q_t-\ell\cdot\Q)}\Vert\geq \epsilon}\right) \leq  d\cdot\exp\left({-\frac{\epsilon^2\ell}{32R^4}}\right),&\\
&\Pr\left({\Vert{\frac{1}{\ell}\M_t}\Vert\geq \epsilon}\right) \leq d\cdot\exp\left({-\frac{\epsilon^2\ell}{32V^2R^2}}\right).&
\end{eqnarray*}
Thus, the lemma follows from applying both of the above bounds with parameter $\epsilon/2$ and noting that $V^2 \leq R^2$.
\end{proof}

We can now finally prove Theorem \ref{thm:main} and Theorem \ref{thm:mainLog}.

\begin{proof}[Proof of Theorem \ref{thm:main}]

The proof follows from straightforward application of the tools we have developed thus-far.

We assume for simplicity that $N = T\cdot\ell$ for our choice of $\ell$. Note this is without loss of generality, since the remainder $(N - \ell\cdot\lfloor{N/\ell}\rfloor)$ affects the bound in the theorem only via lower-order terms.

Let us define $\epsilon := \frac{\delta(\Q)^2 - V^2(V^2+2\lambda_1(\Q))}{72\lambda_1(\Q)}$, and note this choice corresponds to Eq. \eqref{eq:epsilonBound}. Thus, for a certain $\ell = O(R^4\epsilon^{-2}\log\frac{dT}{p})$, we have by an application of Lemma \ref{lem:matHoff} that with probability at least $1-p$ it holds for all $t\in[T]$ that $\frac{1}{\ell}\Vert{\D_t}\Vert \leq \epsilon$. Define a constant approximation parameter $\gamma_t = \gamma = \sqrt{33}(\eta\ell(R+V)^2)^2$ (which corresponds to the bound in Lemma \ref{lem:opmApprox}), where $\eta =\frac{1}{\sqrt{T}\ell(R+V)^2}$ is the fixed chosen learning rate stated in the theorem. Note that for $N$ large enough, all parameters $\epsilon,\eta,\gamma,\hat{\w}_1$ satisfy the conditions of Lemma \ref{lem:goodProj} and Lemma \ref{lem:opmApprox} with probability at least $1-p$, and thus, by invoking Lemma \ref{lem:metaAlgConv} (with $\alpha=0$ and constant $\gamma,\eta$), we have that with probability at least $1-p$ that
\begin{align*}
\lambda_1\left({\sum_{i=1}^N\x_i\x_i^{\top}}\right) - \sum_{t=1}^T\sum_{i=1}^{\ell}(\hat{\w}_t^{\top}\x_t^{(i)})^2 &\leq 
 \frac{1}{\eta} +T\left({\frac{3\sqrt{2}}{2}\frac{\gamma}{\eta} + \frac{\eta}{2}\ell^2(R+V)^4}\right) \\
 &\underset{(a)}{=}  \frac{1}{\eta} +T\left({\frac{3\sqrt{66}}{2}\eta\ell^2(R+V)^4 + \frac{\eta}{2}\ell^2(R+V)^4}\right) \\
 &=  \frac{1}{\eta} +\frac{3\sqrt{66}+1}{2}T\eta\ell^2(R+V)^4 \\
 &\underset{(b)}{=} O\left({\sqrt{T}\ell(R+V)^2}\right)\\
 &\underset{(c)}{=} O\left({\sqrt{N\ell}(R+V)^2}\right),
\end{align*}
where (a) follows from plugging the value of $\gamma$, (b) follows form plugging the value of $\eta$ and (c) follows since $N = T\cdot\ell$. The theorem now follows from plugging-in the bound on $\ell$.
\end{proof}

\begin{proof}[Proof of Theorem \ref{thm:mainLog}]
As in the proof of Theorem \ref{thm:main} we assume for simplicity that $N = T\cdot\ell$ for our choice of $\ell$. We define $\epsilon$ and choose block-length $\ell$ exactly as in the proof of Theorem \ref{thm:mainLog}, which implies that with probability at least $1-p$ it holds for all $t\in[T]$ that $\frac{1}{\ell}\Vert{\D_t}\Vert \leq \epsilon$. We set the regularization parameter to $\alpha = \frac{\ell}{10(\delta(\Q)+V^2)}\left({\delta(\Q)^2-V^4-2V^2\lambda_1(\Q)}\right)$ which agrees with the requirements of Lemma \ref{lem:goodProj} and Lemma \ref{lem:opmApprox}. We set the approximation parameter $\gamma_t$ to $\gamma_t =  \sqrt{33}(\eta_t\ell(R+V)^2)^2$ (which corresponds to the bound in Lemma \ref{lem:opmApprox}). Finally, we set the learning rate on each iteration $t$ to $\eta_t = \frac{1}{\alpha{}t+T_0}$, for
\begin{eqnarray*}
T_0 = \max\{\frac{4\ell\lambda_1(\Q)(V^2+4\epsilon)}{\epsilon},~\ell(R+V)^2,~72\ell\lambda_1(\Q),~\frac{\ell(R+V)^4}{\epsilon}\}.
\end{eqnarray*}
Note that this choice agrees with the requirements of Lemma \ref{lem:goodProj} and Lemma \ref{lem:opmApprox} with respect to both sequences  $\{\eta_t\}_{t\geq 1}$ and $\{\gamma_t\}_{t\geq 1}$. 

Thus, with the choice of initialization $\hat{\w}_1$ stated in the theorem, it follows all parameters $\epsilon,\alpha, \{\eta_t\}_{t\in[T]},\{\gamma_t\}_{t\in[T]},\hat{\w}_1$ satisfy the conditions of Lemma \ref{lem:goodProj} and Lemma \ref{lem:opmApprox} with probability at least $1-p$, and thus, by invoking Lemma \ref{lem:metaAlgConv}, we have that with probability at least $1-p$ that
\begin{align*}
&\lambda_1\left({\sum_{i=1}^N\x_i\x_i^{\top}}\right) - \sum_{t=1}^T\sum_{i=1}^{\ell}(\hat{\w}_t^{\top}\x_t^{(i)})^2 \underset{(a)}{\leq}\\
 & \frac{1}{2}\sum_{t=2}^T\left({(\alpha{}t+T_0)-\alpha - ((t-1)\alpha+T_0)}\right)\Vert{\hat{\w}_t\hat{\w}_t^{\top} - \w\w^{\top}}\Vert_F^2 \\
&+\frac{1}{\eta_1} +\sum_{t=1}^T\left({\frac{3\sqrt{2}}{2}\frac{\gamma_t}{\eta_t} + \frac{\eta_t}{2}( \ell^2(R+V)^4+\alpha^2)}\right) =\\
&\frac{1}{\eta_1} +\sum_{t=1}^T\left({\frac{3\sqrt{2}}{2}\frac{\gamma_t}{\eta_t} + \frac{\eta_t}{2}( \ell^2(R+V)^4+\alpha^2)}\right) \underset{(b)}{=}\\
& (\alpha+T_0) + \left({\frac{3\sqrt{66}\ell^2(R+V)^4}{2} + \frac{\ell^2(R+V)^4+\alpha^2}{2}}\right)\sum_{t=1}^T\eta_t =\\
&(\alpha+T_0) + \frac{(3\sqrt{66}+1)\ell^2(R+V)^4 + \alpha^2}{2}\sum_{t=1}^T\frac{1}{\alpha{}t+T_0} \leq \\
&(\alpha+T_0) + \frac{(3\sqrt{66}+1)\ell^2(R+V)^4 + \alpha^2}{2\alpha}\sum_{t=1}^T\frac{1}{{}t} \leq \\
&T_0+\frac{(3\sqrt{66}+1)\ell^2(R+V)^4 + 3\alpha^2}{2\alpha}\left({1 +\log{}T}\right),
\end{align*}
where (a) follows from Lemma \ref{lem:metaAlgConv} and our choice of learning rate, and (b) follows from our choice of $\{\gamma_t\}_{t\in[T]}$.

Note that 
\begin{eqnarray*}
\alpha = O\left({\frac{\ell}{\delta(\Q)+V^2}\left({\delta^2(\Q) - V^4}\right)}\right) = O\left({\ell(\delta(\Q)+V^2)}\right) = O(\ell(R+V)^2),
\end{eqnarray*}
where in the last equality we have used the fact that $\delta(\Q) \leq \lambda_1(\Q) \leq R^2$. Also, by simple calculations we can see that $T_0 = O\left({\frac{\ell(R+V)^4}{\epsilon}}\right)$. Thus, we have that
\begin{align*}
\lambda_1\left({\sum_{i=1}^N\x_i\x_i^{\top}}\right) - \sum_{t=1}^T\sum_{i=1}^{\ell}(\hat{\w}_t^{\top}\x_t^{(i)})^2 =
O\left({\frac{\ell(R+V)^4}{\epsilon}+\frac{\ell^2(R+V)^4}{\alpha}\log{}T}\right).
\end{align*}
Note that $\alpha = \Theta\left({\frac{\ell\epsilon\lambda_1(\Q)}{\delta(\Q)+V^2}}\right)$ and recall that $\ell = O\left({R^4\epsilon^{-2}\log\frac{dT}{p}}\right)$. Plugging these values we obtain
\begin{align*}
\lambda_1\left({\sum_{i=1}^N\x_i\x_i^{\top}}\right) - \sum_{t=1}^T\sum_{i=1}^{\ell}(\hat{\w}_t^{\top}\x_t^{(i)})^2 &=
O\left({\frac{R^4(R+V)^4\log\frac{dT}{p}}{\epsilon^3}\left({1+\frac{\delta(\Q)+V^2}{\lambda_1(\Q)}\log{}T}\right)}\right) \\
& = O\left({\frac{R^4(R+V)^4(\delta(\Q)+V^2)}{\epsilon^3\lambda_1(\Q)}\log(T)\log\left({\frac{dT}{p}}\right)}\right).
\end{align*}
Finally, plugging the value of $\epsilon$ we obtain the regret bound:
\begin{eqnarray*}
O\left({\frac{R^4(R+V)^4(\delta(\Q)+V^2)\lambda_1(\Q)^2}{\left({\delta(\Q)^2-V^4-2\lambda_1(\Q)V^2}\right)^3}\log(T)\log\left({\frac{dT}{p}}\right)}\right),
\end{eqnarray*}
and the theorem follows.
\end{proof}

\subsection{Proof of Lemma \ref{lem:warmstart} ("warm-start")}
\begin{proof}
Let $\hat{\w}_1$ be the leading eigenvector of the normalized covariance $\hat{\X} = \frac{1}{n}\sum_{i=1}^n\x_i\x_i^{\top}$, where for all $i\in[n]$ $\x_i = \q_i + \v_i$. Clearly, $\E[\hat{\X}] = \Q + \frac{1}{n}\sum_{i=1}^n\v_i\v_i^{\top}$, and thus
\begin{eqnarray*}
\Vert{\Q-\hat{\X}}\Vert^2 \leq 2\Vert{\Q-\E[\hat{\X}]}\Vert^2 + 2\Vert{\hat{\X}-\E[\hat{\X}]}\Vert^2 = 2\Vert{\frac{1}{n}\sum_{i=1}^n\v_i\v_i^{\top}}\Vert^2 + 2\Delta^2,
\end{eqnarray*}
where we use the notation $\Delta = \Vert{\hat{\X}-\E[\hat{\X}]}\Vert$. Via the Davis-Kahan $\sin\theta$ theorem (see for instance Theorem 4 in \cite{G15online}) and using the short notation $\delta = \delta(\Q)$, we have that
\begin{eqnarray*}
(\hat{\w}_1^{\top}\x)^2  = 1 - \frac{1}{2}\Vert{\hat{\w}_1\hat{\w}_1^{\top}-\x\x^{\top}}\Vert_F^2 &\geq &1- 4\frac{\Vert{\hat{\X}-\Q}\Vert^2}{\delta^2}\\
&\geq &1 - 8\frac{V^4}{\delta^2} - 8\frac{\Delta^2}{\delta^2}.
\end{eqnarray*}
Now, using the short notation $\lambda_1=\lambda_1(\Q)$, the requirement 
\begin{eqnarray}\label{eq:WS:req}
(\hat{\w}_1^{\top}\x)^2 \geq 1 - c\frac{\delta-V^2}{2\lambda_1}
\end{eqnarray}
boils down to the condition
\begin{eqnarray*}
8\frac{V^4+\Delta^2}{\delta^2} \leq c\frac{\delta - V^2}{2\lambda_1} ~ \Longleftrightarrow ~
16\lambda_1V^4 + c\delta^2V^2 + 16\lambda_1\Delta^2-c\delta^3 \leq 0.
\end{eqnarray*}
Solving the above inequality on the right for $V^2$, we obtain that \eqref{eq:WS:req} holds when $V^2$ is in the interval:
\begin{eqnarray*}
0 \leq V^2 \leq \frac{-c\delta^2 +\sqrt{c^2\delta^4 + 64\lambda_1c\delta^3-1024\lambda_1^2\Delta^2 }}{32\lambda_1}.
\end{eqnarray*}

In particular, for 
\begin{eqnarray}\label{eq:WS:DeltaBound}
\Delta \leq \frac{\sqrt{19c}\delta^{3/2}}{32\sqrt{\lambda_1}}
\end{eqnarray}
 we obtain that \eqref{eq:WS:req} holds for
\begin{eqnarray*}
0 \leq V^2 \leq \frac{-c\delta^2+\sqrt{c^2\delta^4+45\lambda_1c\delta^3}}{32\lambda_1}.
\end{eqnarray*}

Note that since $c\in(0,1]$ and $\lambda_1\geq\delta$, we have that $\lambda_1c\delta^3 \geq c^2\delta^4$. Note also that $\sqrt{45} > \sqrt{32}+1$. Thus, we have that \eqref{eq:WS:req} holds for $V^2$ in the interval:
\begin{eqnarray*}
0 \leq V^2 \leq \frac{\sqrt{32c\lambda_1\delta^3}}{32\lambda_1} = \frac{\sqrt{c}\delta^{3/2}}{4\sqrt{2}\sqrt{\lambda_1}}.
\end{eqnarray*}


We conclude the proof with the simple observation that using a standard Matrix Hoeffding concentration bound (see for instance Lemma \ref{lem:matHoff}), it suffices to take $n = O\left({\frac{R^4\lambda_1\log(d/p)}{c\delta^3}}\right)$ for the bound in \eqref{eq:WS:DeltaBound} to hold with probability at least $1-p$.

\end{proof}

\section{Experiments}
In this section we present empirical evidence which exhibit the convergence of our proposed algorithms \ref{alg:opm}, \ref{alg:1} in practice and supports our theoretical claims.

We test the following algorithms. Algorithm \ref{alg:1} with block-size $\ell =1$, where $\hat{\w}_{t+1}$ is computed via rank-one SVD (R1-OGA), a similar algorithm which uses non-unit block-size $\ell > 1$ (BR1-OGA), the non-convex online gradient ascent, Algorithm \ref{alg:opm}, with unit block-size $\ell=1$ (Nonconvex-OGA), and the \textit{convex} online gradient ascent algorithm \cite{zinkevich03,Hazan16} (equivalent to Algorithm \ref{alg:1}, but uses accurate Euclidean projections onto the spectrahedron) with unit block-size (Conv-OGA). Since computing that exact projection for Conv-OGA via a full SVD is highly time consuming, we approximate it by extracting only the five leading components. Finally, we record the regret of the initial ``warm-start" vector $\hat{\w}_1$ (BaseVec), which serves as the initialization for all algorithms. For all datasets we plot for each iteration $t$ the average-regret up to time $t$ against the leading eigenvector in hindsight (which is computed w.r.t. all data). For all algorithms introduced in this paper (that is, R1-OGA, BR1-OGA, Nonconvex-OGA) we focus for simplicity on the non-regularized version (i.e., we set $\alpha=0$).

We consider the following three datasets. 

\noindent\textbf{Synthetic:} a random dataset is constructed by generating Gaussian zero-mean data with a random covariance matrix $\Q$ with eigenvalues $\lambda_i = 15\cdot{}0.3^{i-1}$ for all $i\in[d]$, and perturbing them using independent Gaussian zero-mean noise with random covariance matrix $\V$ with eigenvalues $\mu_i = 3\cdot{}0.3^{i-1}$ for all $i\in[d]$, where we use $d=100$. We set the number of data points to $N=10000$, and we compute the initialization $\hat{\w}_1$ for all algorithms by computing the leading eigenvector of a sample of size $100$ (i.e., 1\% of $N$) based on samples from the covariance $\Q$ only. For the algorithm BR1-OGA we set $\ell = 10$. We average the results of 30 i.i.d. experiments.

\noindent\textbf{MNIST:} we use the training set of the MNIST handwritten digit recognition dataset \cite{Lecun1998} which contains 60000 28x28 images, which we split into $N=59400$ images for testing, while $600$ images (i.e., 1\% of data) are used to compute the initialization $\hat{\w}_1$. For the algorithm BR1-OGA we set $\ell = 5$.

\noindent\textbf{CIFAR10:} we use the CIFAR10 tiny image dataset \cite{cifar-10} which contains 50000 32x32 images in RGB format. We convert the images to grayscale and use $N=49900$ images for testing and $100$ images (i.e., 0.2\% of data) are used to compute the initialization. For BR1-OGA we set $\ell = 5$.

The results for all three datasets are given in figure \ref{fig:1}. It can be seen that indeed all algorithms improve significantly over the ``warm-start" base vector. We also see that all algorithms indeed attain low average-regret, and in particular are competitive with OGA which follows a convex approach (up to the approximation of the projection via thin SVD).

To further examine the applicability of our theoretical approach, for all datasets, we recorded for algorithm BR1-OGA the fraction of projection errors, i.e., the precent of number of iterations $t$ on which the projection of the matrix $\W_{t+1}=\hat{\w}_t\hat{\w}_t^{\top}+\eta\X_t$ onto the spectrahedron $\mS$ is not a rank-one matrix. The results are 6.24\%, 0.26\%, 0\%, for synthetic, MNIST and CIFAR10, respectively. These low error rates indeed support our theoretical analysis which hinges on showing that under our data model (recall Assumption \ref{ass:dist}) and given a "warm-start'" initialization, the projections of the matrices $\W_t$ in Algorithm \ref{alg:1} are always rank-one.

\begin{figure*}
    \centering
    \begin{subfigure}{0.32\textwidth}
        \includegraphics[width=2.13in]{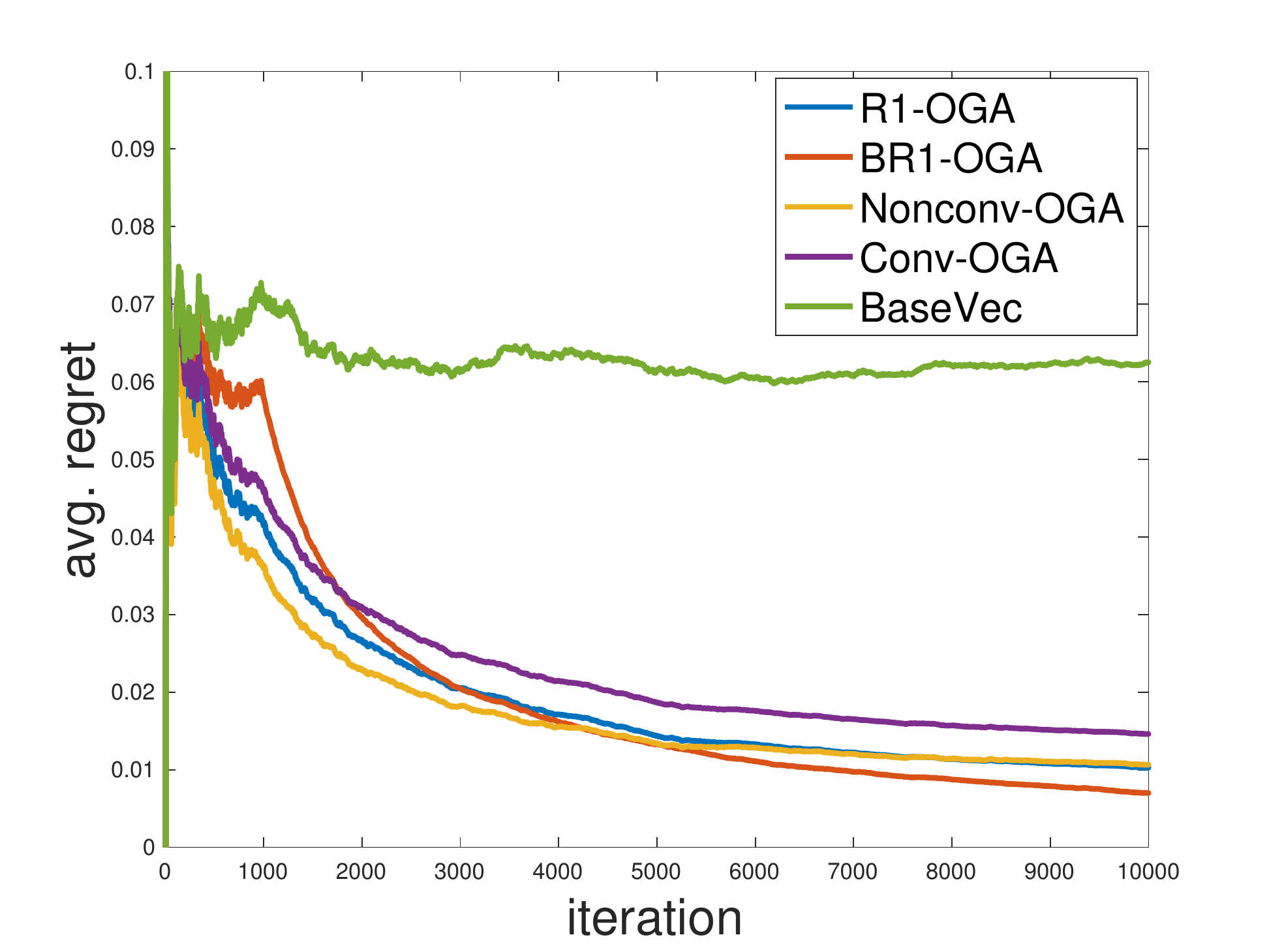}
        \caption{synthetic}
     \end{subfigure}
          \begin{subfigure}{0.32\textwidth}
        \includegraphics[width=2.13 in]{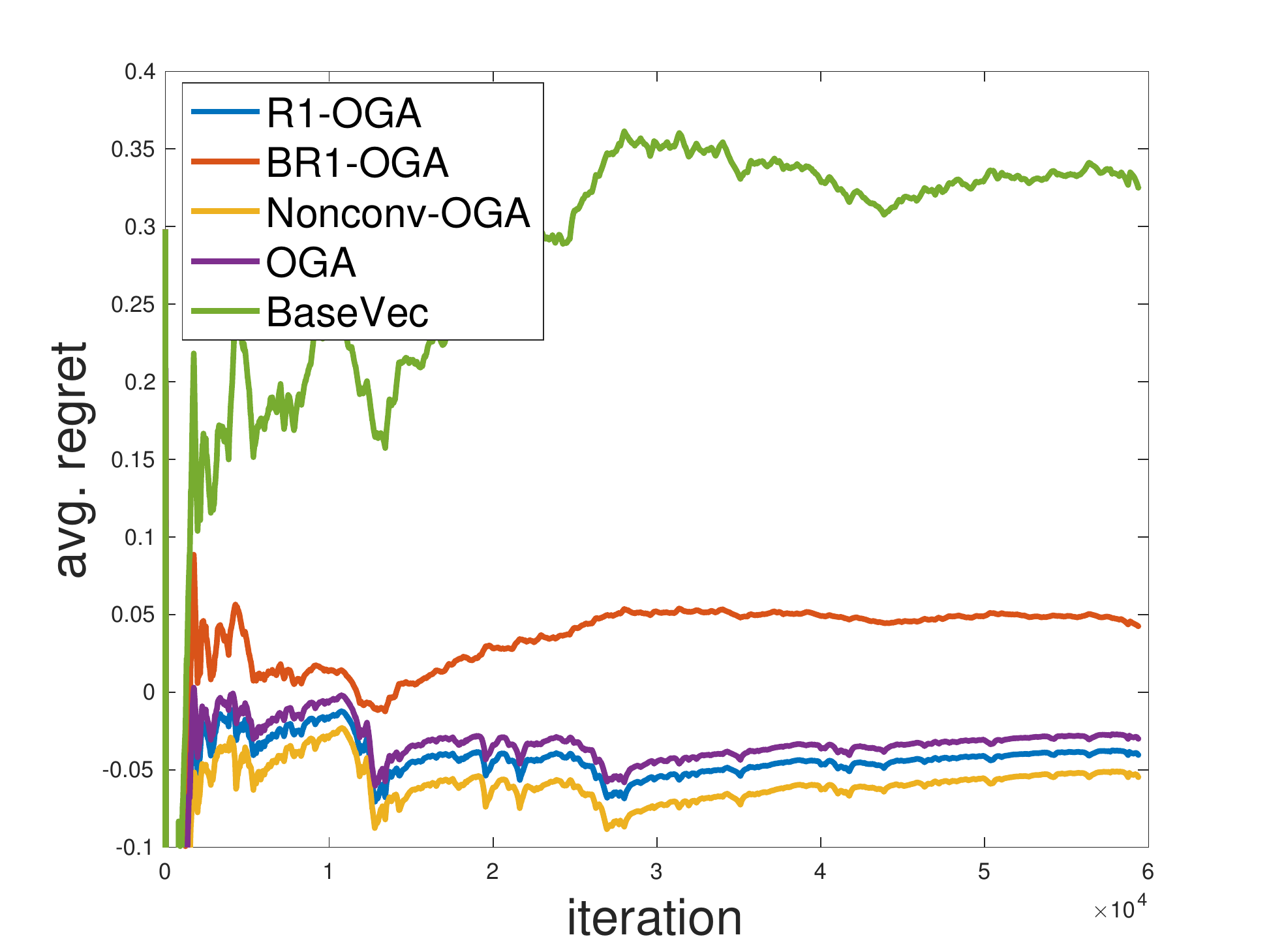}
        \caption{MNIST}
      \end{subfigure}
    \begin{subfigure}{0.32\textwidth}
        \includegraphics[width=2.13in]{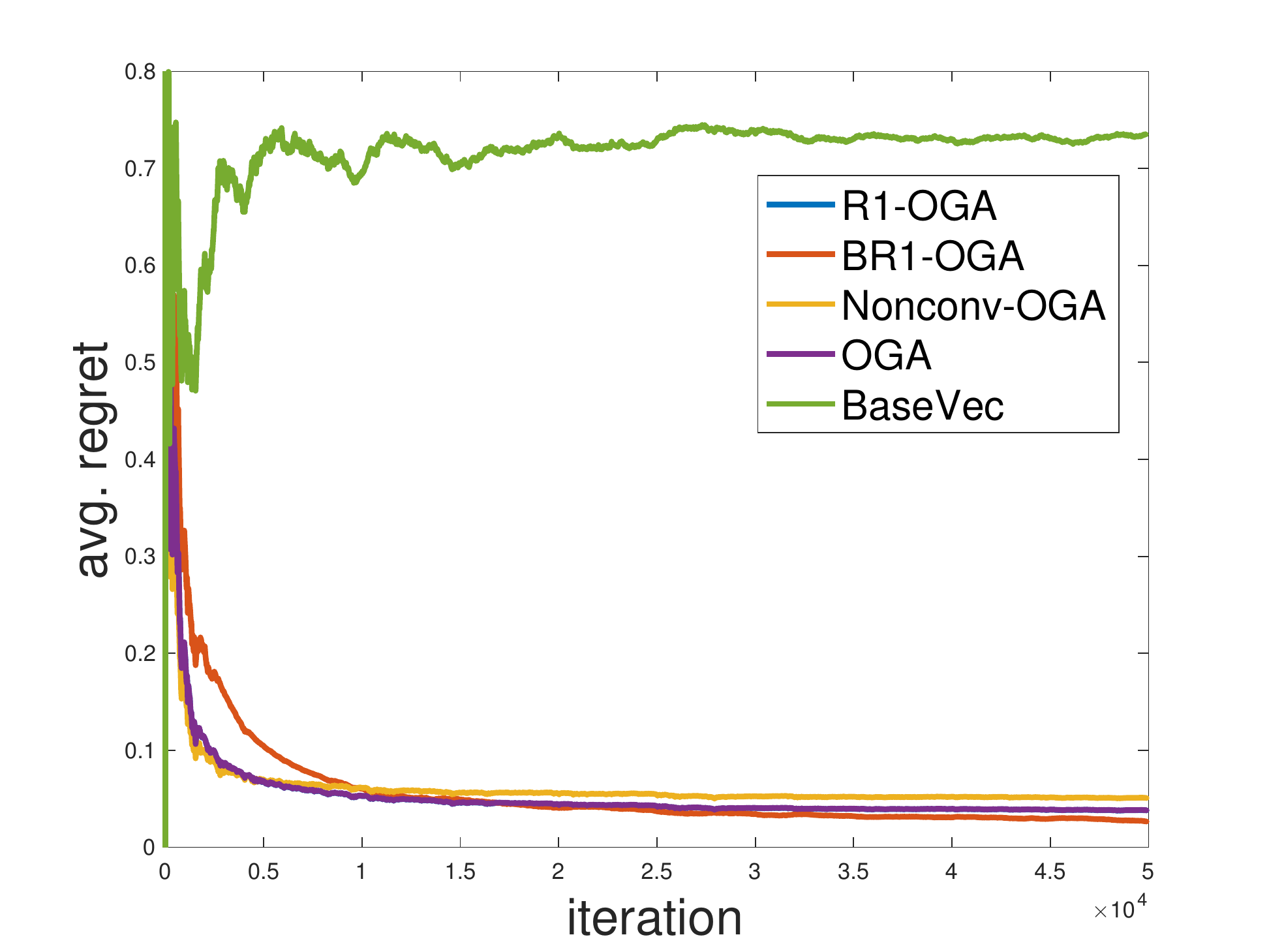}
        \caption{CIFAR10}
          \end{subfigure}    
    \caption{Average-regret of tested algorithms.}\label{fig:1}
\end{figure*}

\section{Discussion}
In this paper we took a step forward towards understanding the ability of highly-efficient non-convex online algorithms to minimize regret in adversarial online learning settings. We focused on the particular problem of online principal component analysis with $k=1$, and showed that under a "semi-adversarial" model, in which the data follows a stochastic distribution with adversarial perturbations, and  given a ``warm-start" initialization, the natural nonconvex online gradient ascent indeed guarantees sublinear regret. Our theory is further supported by empirical evidence.

We hope this work will motivate further research on online nonconvex optimization with global convergence guarantees. Future directions of interest may include extending our analysis to a wider regime of parameters, and extracting $k$ principal components at once. Also, it is interesting if in the standard adversarial setting, it can be shown that online nonconvex gradient ascent achieves low-regret, or on the other-hand, to show that there exist instances on which it cannot guarantee non-trivial regret. 
Finally, moving beyond PCA, other online learning problems of interest that may benefit from a non-convex approach include online matrix completion \cite{Hazan12, jin2016provable}, and of course, provable online learning of deep networks.

\bibliographystyle{plain}
\bibliography{bib}

\end{document}